\documentclass[letterpaper, 11pt]{amsart}

\usepackage{amsmath,amssymb,amsxtra,amsthm,hyperref} 
\usepackage{ulem}
 \usepackage[margin=1in]{geometry}
\usepackage{tikz}
\usetikzlibrary{calc,decorations.markings}
\usetikzlibrary{decorations.pathreplacing,calligraphy}
\usetikzlibrary{arrows,positioning, shapes.symbols,shapes.callouts,patterns}

\usepackage{standalone}

\usepackage{multicol}
\usepackage{enumitem}

\newcommand{\Bmath}[1]{\mbox{\bf {#1}}}

\usepackage{booktabs}

\def\c{{\Bmath c}}

\usepackage{thmtools}
\usepackage{thm-restate}

\newtheorem{theorem}{Theorem}[section]

\newtheorem{corollary}[theorem]{Corollary}
\newtheorem{proposition}[theorem]{Proposition}   
 
\newtheorem{lemma}[theorem]{Lemma}

\theoremstyle{definition}
\newtheorem{definition}[theorem]{Definition}

\newtheorem{remark}[theorem]{Remark}

\numberwithin{equation}{section}

\usepackage{xcolor}
\usepackage{soul}

\newcommand{\norm}[1]{\left\lVert#1\right\rVert}

\usepackage{soul}

\begin{document}

\title{Symplectic Representations of Legendre Dynamics}
\author{Robert Simon Fong}
\email{rsfong@g.ecc.u-tokyo.ac.jp}
\address{International Research Center for Neurointelligence, The University of Tokyo, Tokyo, 113-0033, Japan}

\author{Gouhei Tanaka}
\email{gtanaka@nitech.ac.jp}
\address{Department of Computer Science, Graduate School of Engineering, 
        Nagoya Institute of Technology, Nagoya, 466-8555, Japan}
       
\author{Kazuyuki Aihara} 
       \email{kaihara@g.ecc.u-tokyo.ac.jp}
\address{International Research Center for Neurointelligence, The University of Tokyo, Tokyo, 113-0033, Japan}
       

\begin{abstract}

  {
Modern learning systems act on internal representations of data, yet how these representations encode underlying physical or statistical structure is often left implicit. In physics, symplecticity keeps Hamiltonian systems faithful to their phase-space geometry. Recent learning methods impose such geometric structure either in the dynamics or through training losses. Here we ask a different question: what would it mean for the representation itself to obey a symplectic conservation law?

We pose this representation-level constraint through Legendre duality: the relation $p = d\psi(q)$ between primal and dual coordinates, which in exponential family models is the information-geometric pairing of natural and expectation parameters. We formalize Legendre dynamics as stochastic processes whose trajectories remain on Legendre graphs, where the evolving primal-dual parameters stay Legendre dual. We show that this class includes linear time-invariant Gaussian process regression and Ornstein-Uhlenbeck dynamics.

Geometrically, we characterize the symplectomorphisms of cotangent bundles that preserve all Legendre graphs. We show that these maps are exactly cotangent lifts of base diffeomorphisms followed by exact fibre translations. This gives an explicit normal form for Legendre-preserving representation updates.

Dynamically, we prove that the normal form is realized by Hamiltonians that are at most linear in the momentum. This realization principle is used to construct linear and nonlinear Hamiltonian Symplectic Reservoirs (SR) whose recurrent updates preserve Legendre graphs by construction. This is the only normal form that preserves Legendre duality, so the architecture follows from the invariant. Numerical experiments confirm the normal-form identities and distinguish Legendre preserving Hamiltonian SRs from generic symplectic and standard reservoir baselines.

}
\end{abstract}

\maketitle

\smallskip
\noindent \textbf{Keywords.} 
  Information Geometry, Symplectic Geometry, Representation Learning, Hamiltonian Mechanics, Reservoir Computing

\section{Introduction}
 
In machine learning, inputs live in a structured space and are mapped through a chain of computation into outputs. Inputs are first embedded into a feature space of representations (such as $\mathbb{R}^n$), and outputs are then formed on top of these representations. This raises a basic question: what makes a representation ``good''? A broad line of work views good representations as those that encode the task-relevant invariants and equivalence relations, such as the metric structure and symmetries  \cite{bengio2013representation,anselmi2016invariance}. Metric learning is one example of this kind, making the geometry of a chosen distance explicit  \cite{kulis2013metric}. These examples rest on a common premise that input spaces are rarely just bare sets: they carry \textbf{intrinsic structure}, {such as} metrics, group actions, {and/or} invariance relations, and a representation earns its keep by preserving it. 

This viewpoint is familiar in vision: {convolutional layers are good image representations because they are discrete shift-equivariant. Together with pooling and data augmentation this yields approximate translation invariance, where group-equivariant extensions further build in rotations and other symmetries  \cite{cohen2016group,kayhan2020translation}. This improves sample efficiency and reduces spurious coordinate dependence.}  {By generalizing} this viewpoint, we regard a representation as \textbf{``good''} if it preserves the structure of the input space in the feature space.
We summarize the above as a \textbf{chain of structure-preserving maps} {as illustrated in the following diagram}:
\[
\boxed{ \left\{ \begin{array}{c} \text{input space}, \\ \text{structure} \end{array} \right\} \xrightarrow[\text{}]{\text{representation}} 
\left\{\begin{array}{c} \text{feature space}, \\ \text{structure} \end{array} \right\} }
\xrightarrow[\text{}]{\text{readout}} 
\left\{\begin{array}{c} \text{output space}, \\ \text{structure} \end{array} \right\} 
\].

Our focus is the representation stage of the chain, shown in the boxed region of the diagram. In particular, we study when a learned representation preserves the intrinsic structure attached to the domain one wishes to model. 


We formalize this structure-preservation viewpoint as follows. A structured space $\left(X,S_X\right)$  is defined by the underlying set $X$ and a family of relations $S_X=\left\{\sim_X^\alpha\right\}_{\alpha\in A}$  {indexed by set $A$} (e.g., metric closeness at given radii, symmetry-equivalence, membership in a constraint set). A map $R:X\rightarrow Z$ {is said to} preserve structure if, for all $\alpha \in A$ and all $x_1, x_2 \in X$, 
\[
x_1 \sim_X^\alpha x_2 \implies R(x_1) \sim_Z^\alpha R(x_2).
\]
This forward preservation condition is the notion of structure preservation
used throughout the paper. Intuitively, a good representation preserves the structure of the tuple $\{\text{input space},\text{structure}\}$.

Our setting is inference on exponential-family models, whose parameter spaces admit dual coordinate systems generated by a Legendre potential $\psi$  \cite{amari2000methods}. Concretely, we work on the parameter manifold $Q$ with Legendre duality encoded by the Legendre graph:
\begin{align}
    \label{eqn:intro_legendre}
    L_\psi = \left\{(q,p) \in T^*Q : p = d\psi(q)\right\}.
\end{align}

We call parameter evolutions that remain on (possibly drifting) Legendre graphs \textbf{Legendre dynamics} (LD). {We show that} this class includes linear time-invariant Gaussian process regression (LTI-GPR) and Ornstein-Uhlenbeck (OU) dynamics, where their Gaussian information-form updates can be expressed so that the parameters $\left(q_k,p_k\right) = \left(q_k,{{d}}\psi(q_k)\right)$ remain dual at each step $k$, with the potential $\psi_k$ allowed to change. To the best of our knowledge, this explicit ``Legendre-graph-preserving'' viewpoint has received little {direct} attention in previous work and will be central in what follows.

To be concrete, the raw inputs are left-infinite sequences $u=\{u_k\}_{k\in\mathbb Z_-}$. The structured representation considered in this paper is a cotangent-bundle state
\[
x_k=(q_k,p_k)\in T^*Q.
\]
The preservation results below apply to this representation: if $x_k\in L_{\psi_k}$, then the Hamiltonian symplectic reservoir update maps $x_k$ to a point $x_{k+1}\in L_{\psi_{k+1}}$. A downstream task readout
\[
h:T^*Q\to Y
\]
may then be trained on top of $x_k$. The preservation result concerns the recurrent state, so a readout carries structure forward only if it is itself structure preserving.


One of the machine-learning frameworks suited to processing time-dependent data streams is Reservoir Computing (RC), which employs a fixed reservoir for input transformation together with a trainable readout  \cite{Lukoservicius2009,tanaka2019recent}. Structure-preserving neural networks and physical substrates can be used to build reservoirs that approximate dynamical systems while respecting prescribed conservation laws. 

 {
Motivated by Hamiltonian mechanics, we consider \textbf{Hamiltonian Symplectic Reservoirs} (Hamiltonian SR): input-driven recurrent systems on $T^*Q$ whose updates are Hamiltonian symplectomorphisms. This motivates our representation-level question:


\begin{center}
    	\textit{What does symplectic conservation mean at the representational level?}
\end{center}

Our answer is geometric. For Legendre dynamics, the structure to be preserved is the tuple
\[
\{\text{Parameter space }Q,\ \text{Legendre duality}\},
\]
which is encoded by the family of Legendre graphs of Equation~\eqref{eqn:intro_legendre}.
The Hamiltonian SR updates constructed in this paper are Legendre preserving as they map each Legendre graph to another Legendre graph. Therefore, if the current cotangent-bundle representation lies on $L_{\psi_k}$, then one update sends it to some $L_{\psi_{k+1}}$. The potential may change, but the primal-dual relation is preserved. By induction, a trajectory initialized on a Legendre graph thus remains within the family of Legendre graphs. This provides a faithful, invariance-preserving representation for Legendre dynamics such as LTI-GPR and OU. In this paper, the main normal-form theorem characterizes the Legendre preserving symplectomorphisms, and the Hamiltonian SR constructions realize this class of structure preserving representations.

Modelling Legendre dynamics has two practical payoffs. On the statistical level, Legendre duality encodes the natural/expectation parameter pairing and the Fisher-geometric structure of an exponential family. Preserving the Legendre relation keeps uncertainty, Fisher geometry, and dual coordinates mutually consistent along the trajectory
\[
(q_k,p_k=d\psi_k(q_k)).
\]
On the dynamical level, the same Legendre relation is the intrinsic cotangent-bundle structure underlying Hamiltonian and thermodynamic descriptions, where potentials generate conjugate variables. Preserving Legendre graphs therefore gives a representation-level way to track how potentials, curvature, and dual coordinates evolve under input-driven dynamics.

This places our approach in a distinct part of the surrounding landscape. Physics-informed neural networks enforce conservation laws through losses on the outputs  \cite{raissi2019physics}, while symplectic, Hamiltonian, and Poisson learning methods impose geometric structure on the dynamics  \cite{chen2019symplectic, jin2020sympnets, eldred2024lie}. Our viewpoint is complementary to both: we place the geometric structure in the recurrent representation map itself, designing the update so that it is symplectic and Legendre preserving. The recurrent core thus aligns the internal representation with the invariant (Legendre duality) of the problem, in analogy with how convolutional layers align feature maps with spatial symmetries in vision. Within RC  \cite{Jaeger2001}, this gives the reservoir a specific geometric role instead of treating it as an unstructured dynamical system.

 }


The formulation of this paper is parametrization-invariant or coordinate-free. That is, all statements are valid under reparameterization. This aligns with the physics principle that laws should be independent of coordinates, and the representation of Legendre dynamics therefore brings us closer to the intrinsic structure of the problem.

\subsection{Related Work}

Representation learning for time series and streaming data has been explored across many model classes, including RNNs  \cite{trirat2024universal}, yet the theoretical analysis of their representation capacity remains limited. Our aim is to deepen the theoretical understanding of the representational capacity of a special class of RNNs that exhibit dynamic state evolution while explicitly preserving structure.


Among various RNN-based RC models  \cite{Jaeger2001,Maass2002,Tino2001,Lukoservicius2009}, the Echo State Network (ESN) is the simplest and most widely studied  \cite{Jaeger2001,Jaeger2002,jaeger2002a,Jaeger2004}. The universal approximation capability of ESNs has been demonstrated in a series of papers, showing existentially that ESNs can approximate any time-invariant fading-memory dynamic filter in a variety of settings   \cite{Grigoryeva2018, grigoryeva2018echo, gonon2019reservoir}. The trade-off between memory capacity and nonlinear expressiveness in the reservoir was pointed out  \cite{dambre2012information}, and further studied in  \cite{inubushi2017reservoir}, which adjusted the degree of reservoir nonlinearity to reconcile the two. Typical ESNs use a randomly connected RNN as the reservoir, which inevitably generates fluctuations in reservoir's expressiveness and their task performance. One way to mitigate this issue is to use a structured reservoir, as in Euler State Networks (EuSN), which enforce antisymmetric Euler-discretized dynamics to obtain near-orthogonal recurrent matrices for memory control and stability  \cite{gallicchio2024euler}.


Linear ESNs have recently been studied extensively owing to their mathematical tractability. In statistical time-series modelling, they have been shown to have representational capacity equivalent to that of a nonlinear vector autoregression (NVAR)  \cite{bollt2021explaining}.
From the kernel viewpoint, a linear ESN can be interpreted as a kernel machine operating in a temporal feature space, where inputs are mapped by a linear reservoir  \cite{Tino_JMLR_2020}. The same work also showed that a linear simple cycle reservoir (SCR), a ring-topology neural network with minimal complexity  \cite{rodan2010minimum}, can implement deep-memory dynamic kernels. A subsequent study revealed that, at the edge of stability, the kernel representation of SCR replicates the Fourier decomposition, providing a natural link between reservoir-based signal processing and classical signal-processing models  \cite{fong2025linear}. Moreover, it was shown that the SCRs, together with a trained linear readout, are universal approximators of time-invariant dynamic filters with fading memory over $\mathbb{C}$ and $\mathbb{R}$ respectively  \cite{li2023simple, fong2024universality}.


 {
Several recent works learn Hamiltonian, symplectic, or Poisson structure directly. Symplectic recurrent neural networks unroll neural Hamiltonian dynamics using symplectic integration, while SympNets construct
feedforward modules that approximate symplectic maps
 \cite{chen2019symplectic,jin2020sympnets}. Lie--Poisson neural networks and coupled Lie--Poisson neural networks preserve Poisson brackets, Casimirs, and momentum-type invariants in symmetry-reduced Hamiltonian systems  \cite{eldred2024lie,eldred2025clpnets}. Other structure-preserving learning methods recover Hamiltonian or Poisson vector fields using kernel methods, or learn variational/Lagrangian systems using Gaussian processes
 \cite{hu2024structure,hu2025global,
offen2024machine}. Related geometric approaches also use Lagrangian submanifolds to enforce symmetry or momentum preservation in Hamiltonian learning problems  \cite{vaquero2024symmetry}.

These works preserve symplectic, Poisson, Hamiltonian, variational, or momentum-type structures. The structural invariant studied here is Legendre duality, encoded by the Legendre graph $L_\psi=\{(q,p):p=d\psi(q)\}$ of Equation~\eqref{eqn:intro_legendre}. This invariant is intrinsic to stochastic parameter dynamics such as LTI-GPR and OU, where it records the coupling between primal and dual coordinates. The normal form from our main theorem
\[
\tau_{d\chi}\circ f^\sharp
\]
thus  characterizes the cotangent-bundle updates that are Legendre preserving, and the Hamiltonian SR constructions below realize this class recurrently. This provides a representation level guarantee for the geometry of primal-dual parameter dynamics.


}

\subsection{Contributions}

The development proceeds in three steps. We first ask which symplectic maps preserve Legendre duality, a question of geometry. We then ask which Hamiltonians generate those maps, a question of dynamics. We finally construct the reservoirs that realize them, a question of neural networks. The contributions of this paper and its contents are outlined as follows:

\begin{itemize}
  \item In Section~\ref{sec:domain}, we introduce the notion of Legendre dynamics (LD): stochastic processes whose parameter trajectories evolve on (possibly time varying) Legendre graphs $L_{\psi_k} = \{(q,p) : p = {{d}}\psi_k(q)\}$. We show that this framework covers, in particular, linear time-invariant Gaussian process regression (LTI-GPR) and Ornstein-Uhlenbeck (OU) processes under a single structural invariant, namely Legendre duality between primal and dual coordinates.

  \item In Section~\ref{sec:duality_lagrangian} and Section~\ref{sec:preservation}, we derive a geometric characterization of symplectomorphisms that preserve Legendre graphs. Viewing each $L_{\psi}$ as a Lagrangian submanifold of $(T^{*}Q,\omega)$, we prove that a symplectomorphism preserves Legendre type structure if and only if it has the normal form $\tau_{d\chi_t} \circ f_t^{\sharp}$, i.e., an exact fiber translation by $d\chi_t$ composed with a cotangent lift of a base diffeomorphism $f_t:Q\to Q$. This provides an \textbf{if and only if} (necessary and sufficient) structural classification of the symplectic representation updates that preserve Legendre dynamics.
   {
  \item In Section~\ref{sec:linearHamiltonian_src},   we introduce Hamiltonian realizations of the Legendre preserving normal form. We first derive generic linear symplectic reservoirs as a diagnostic comparison, then introduce a linear realization and a nonlinear Hamiltonian Symplectic Reservoir generated by Hamiltonians that are at most linear in the momentum. These realizations are Legendre preserving and provide invariance-preserving representations for Legendre dynamics.
  \item In Section~\ref{sec:exp}, we provide numerical illustrations of the structural preservation claims. The experiments verify the local normal-form identities and the trajectory-wise transport of Legendre duality under OU-driven inputs, and compare the N-HSR against generic symplectic and standard orthogonal reservoir baselines.
}
\item
  In Appendix~\ref{app:OU_dynamics}, we develop a Markov-semigroup formulation of Legendre dynamics on exponential families. We show that a Markov semigroup preserves an exponential family if and only if its generator acts affinely on the sufficient statistics, and we apply this criterion to prove that the Ornstein-Uhlenbeck process is an example of Legendre dynamics.
 
\end{itemize}

For the remainder of the paper we adopt the Einstein summation convention where repeated indices are summed over their range unless explicitly stated otherwise.

\section{Domain of discourse: Legendre dynamics}
\label{sec:domain}

This section formalizes the $\left\{\text{input space}, \text{structure}\right\}$ of the main diagram:

\[
\boxed{\left\{\text{input space}, \text{structure}\right\}} \xrightarrow[\text{}]{\text{representation}} 
\left\{\text{feature space}, \text{structure}\right\}.
\]

Many stochastic processes and Bayesian update rules admit an information-geometric description: parameters evolve on a manifold $(Q,\psi)$ equipped with potential $\psi$, and remain on the Legendre graph (Equation~\eqref{eqn:intro_legendre}):
\[
L_\psi = \left\{(q,p): p = {{d}}\psi(q)\right\}.
\]

We call such update rules \textbf{Legendre dynamics} (LD). We show that this abstraction unifies models such as linear time-invariant Gaussian process regression (LTI-GPR) and Ornstein-Uhlenbeck (OU) processes under a single structural invariant: Legendre duality between primal and dual coordinates on the parameter manifold $Q$. LD thus captures the geometrical invariance underlying statistical inference, independent of any particular metric or coordinate system. This isolates the structure that our representations are designed to preserve. We now recall the basic Legendre geometry following  \cite{amari2000methods}:

\begin{definition}
    Let $\mathcal{P}:= \left\{ p_\theta \middle| \theta\in\Xi\right\}$ denote a statistical model on a measurable space $\mathcal{U}$ parametrized by $\Xi$.
    A \textbf{dualistic model} is a tuple $\left(\mathcal{P}, \psi\right)$ where $\psi$ is a $C^2$ and strictly convex \textbf{potential} function such that the following gradient map induces a diffeomorphism between natural and dual coordinates
    \[
    \bar\eta(\theta)=\nabla\psi(\theta)    .
    \]
\end{definition}

\begin{definition}
\label{defn:Legendre_dynamics_SP}
    Let $\left(U_k\right)_{k\geq 0}$ denote a stochastic process on $\mathcal{U}$ with laws $\mu_k = \operatorname{Law}\left(U_k\right)$. We say that $\left(U_k\right)_{k\geq 0}$ is a \textbf{(discrete-time) Legendre dynamics} on the dualistic model $\left(\mathcal{P}, \psi\right)$ if there exists parameter maps:
    \begin{align*}
        g_k: \Xi \rightarrow \Xi
    \end{align*}
    such that for every $k\geq 0$ and $\theta_k\in \Xi$ with $\mu_k$ having density $p_{\theta_k} \in \mathcal{P}$:
    \begin{itemize}
        \item \textbf{Closure}: the next law is again in $\left(\mathcal{P}, \psi\right)$:
        \[
        \theta_{k+1} = g_k\left(\theta_k\right), \quad \text{ where } \mu_{k+1} \text{ has density } p_{\theta_{k+1}}\in \mathcal{P}.
        \]
        
        \item \textbf{Duality Preservation}: there exists a potential function $\psi'$ such that dual coordinates at step $k+1$ are given by
        \[
        \bar\eta_{k+1} = \nabla\psi'(\theta_{k+1})
        \]
        If $\psi'=\psi$, we call the dynamics \textbf{fixed potential} Legendre dynamics. Otherwise, we call it \textbf{general} Legendre dynamics. Thus fixed potential Legendre dynamics is a special case of general Legendre dynamics.
       
    \end{itemize}
    Equivalently, the update $g_k$ preserves the parameter Legendre duality of dualistic model $\left(\mathcal{P}, \psi\right)$. Note that in the data-assimilation setting we allow $g_k$ to depend on the realized observation $y_k$ as well.
\end{definition}

\begin{remark}
    For the rest of the section, when we specialize to the Gaussian exponential family, the Gaussian natural parameters will be denoted by $(\xi,\Lambda)$. The expectation/dual coordinate will be denoted by $\bar\eta(\theta)=\nabla\psi(\theta)$.

\end{remark}

We now show that the class of LTI Gaussian process regression fits into this Legendre dynamics framework.

\begin{theorem}
Let $\mathcal{P}$ be a Gaussian exponential family on $\mathbb{R}^d$. A Gaussian process regression whose prior covariance function can be expressed as the solution to a linear time-invariant stochastic differential equation is (fixed potential) discrete-time Legendre dynamics on the dualistic model $\left(\mathcal{P},\psi\right)$.
\end{theorem}

\begin{proof}
Consider Gaussian exponential family on $\mathcal{U}\subseteq \mathbb{R}^d$ given by 
\[
\mathcal{P} =\left\{p_{\xi, \Lambda}(u)= \exp\left(\xi^\top u - \frac{1}{2}u^\top \Lambda u  - \psi\left(\xi,\Lambda\right)\right)\right\}
\]
with natural parametrization $\theta:= (\xi,\Lambda) \in \Xi$ where $\xi\in\mathbb{R}^d$ and $\Lambda \in \operatorname{Sym}_d^{++}$. For this family the sufficient statistics is $T(u) := \left(u, -\frac{1}{2}uu^\top \right)$.
The potential function $\psi(\xi,\Lambda)$ is given by 
\begin{align}
    \label{eqn:GP_potential}
    \psi(\xi,\Lambda) = \frac{1}{2} \xi^\top\Lambda^{-1}\xi - \frac{1}{2} \log \det\Lambda + \frac{d}{2} \log(2\pi).
\end{align}
The dual parameter of $\mathcal{P}$ is given by
\begin{align}
\label{eqn:eta_bar}
\bar{\eta}(\theta) = \mathbb{E}_{\theta}\left[T(u)\right]= \nabla\psi\left(\xi,\Lambda\right) = \left(m,-\frac{1}{2}\left(\Sigma + mm^\top\right)\right),
\end{align}
where $m := \Lambda^{-1}\xi$, $\Sigma = \Lambda^{-1}$.

Let $U_k$ denote the state vector of a Gaussian process at time $k\geq 0$. Since the Gaussian process prior is a solution to a linear time-invariant stochastic differential equation (LTI-SDE), it admits the following state space representation, where one Gaussian process regression step can be written in state-space (Kalman) form by  \cite{sarkka2023bayesian,hartikainen2010kalman}:
\begin{itemize}
    \item \textbf{(Prediction)}:
    \begin{align}
    \label{eqn:GPR_predict}
        U_{k+1} = AU_k + \zeta_k, \quad \zeta_k \sim\mathcal{N}(0, Q_\zeta),
    \end{align}
    where $A$ is the state transition matrix and $Q_\zeta$ is symmetric positive definite.
    \item \textbf{(Update with observation $y_k$)}:
    \begin{align}
    \label{eqn:GPR_update}
        y_{k} = HU_{k} + \epsilon_{k},\quad \epsilon_k\sim\mathcal{N}(0,R)
    \end{align}
    where $H$ is the observation matrix  and $R$ is symmetric positive definite.
\end{itemize}
Note that $\zeta_k$ and $\epsilon_k$ are independent of each other and of $U_k$.

We now show that the induced parameter update of the above two-step process {is} given by
\begin{align}
    \label{eqn:g_k_GPR}
g_k\left(\cdot; y_k\right) = \Phi_{\operatorname{upd}}\left(\cdot; y_k\right) \circ \Phi_{\operatorname{pred}}: \Xi\rightarrow\Xi,
\end{align}
which satisfies the definition of fixed potential discrete time Legendre dynamics. 
First let's show that the prediction step of Gaussian  {process} regression satisfies the definition of Legendre dynamics. By Equation~\eqref{eqn:GPR_predict}, the state vector $U_{k+1}$ has an affine dependence on $U_k$. In particular, if $U_k \sim \mathcal{N}(m,\Sigma)$, then by construction:
\[
U_{k+1}\sim \mathcal{N}(m', \Sigma'), \quad m' = Am, \quad \Sigma' = A\Sigma A^\top  + Q_\zeta.
\]
Substituting the above into natural coordinates, the state transition is given by
\[
\left\{\!\begin{aligned}
\Lambda &= \Sigma^{-1} \\
\xi &= \Sigma^{-1}m
\end{aligned}\right\} 
\mapsto
\left\{\!\begin{aligned}
\Lambda' &= \left(\Sigma'\right)^{-1} = \left( A\Sigma A^\top + Q_\zeta \right)^{-1} \\
\xi' &=\left(\Sigma'\right)^{-1}m' =  \Lambda' Am =\Lambda'A\Lambda^{-1}\xi 
\end{aligned}\right\} .
\]
Hence the induced parameter update map of the prediction step is
\begin{align*}
    \Phi_{\operatorname{pred}}:\Xi&\rightarrow \Xi \\
    \theta & \mapsto \theta':=\left(\xi',\Lambda'\right),
\end{align*}
which satisfies the closure requirement of Definition~\ref{defn:Legendre_dynamics_SP}. By definition, the updated dual parameter is obtained by the gradient of Equation~\eqref{eqn:GP_potential}. Since by construction we have $m' = (\Lambda')^{-1}\xi'$ and together with the exponential family identity $\mathbb{E}_\theta\left[T(u)\right]= \nabla\psi(\theta)$  and the Gaussian sufficient statistics being $T(u) = \left(u,-\frac{1}{2}uu^\top\right)$, we have the following:
\begin{align*}
    \bar{\eta}'(\theta') &:= \mathbb{E}_{\theta'}\left[T(u)\right] = \left(m', -\frac{1}{2}\left(\Sigma' + (m')(m')^\top\right)\right) = \nabla\psi(\theta').
\end{align*}
Hence $\Phi_{\operatorname{pred}}$ satisfies the fixed potential duality preservation condition of Definition~\ref{defn:Legendre_dynamics_SP}. 

Now by Equation~\eqref{eqn:GPR_update}, given observation $y_k$ the Gaussian likelihood in scaled form is

\begin{align}
    \label{eqn:likelihood}
    p(y_k |u) &\propto \exp\left(-\frac{1}{2}\left(y_k - Hu\right)^\top R^{-1} \left(y_k - H u\right) \right) \nonumber \\ 
    & = \exp\left( H^\top R^{-1} y_k \cdot u - \frac{1}{2}u^\top \left(H^\top R^{-1}H \right) u \right).
\end{align}
Hence the natural parameters of $p(y_k |u)$ are given by 
\[
\left(\xi_L, \Lambda_L\right) := \left(H^\top R^{-1} y_k,  H^\top R^{-1}H \right).
\]
Recall $\theta' = \left(\xi', \Lambda'\right)$ from $\Phi_{\operatorname{pred}}$. The posterior density is given by
\[
p(u|y_k) = \frac{p(y_k | u)\cdot p_{\theta'}(u)}{p(y_k)} \propto p(y_k | u)\cdot p_{\theta'}(u).
\]
Since multiplication of exponentials just adds up the linear and quadratic terms in $u$, $p(u|y_k)$ is again Gaussian with natural parameters:
\[
\begin{cases}
    \xi'' &= \xi_L + \xi'  = \xi' + H^\top R^{-1} y_k \\
    \Lambda'' &= \Lambda_L + \Lambda' = \Lambda' + H^\top R^{-1}H.
\end{cases}
\]
Since $R$ is positive definite, for any matrix $H$, $H^\top R^{-1}H$ is positive semi-definite. Since $\Lambda'$ is positive definite and $H^\top R^{-1}H$ is positive semi-definite, by construction $\Lambda''$ is positive definite. Hence $p(u|y_k) \in \mathcal{P}$ with natural parameters $\theta'' = \left(\xi'', \Lambda''\right)$. This means that the induced parameter change map of the update step, given by
\begin{align*}
    \Phi_{\operatorname{upd}}:\Xi&\rightarrow \Xi \\
    \theta' & \mapsto \theta'':=\left(\xi'',\Lambda''\right),
\end{align*}
satisfies the closure requirement of Definition~\ref{defn:Legendre_dynamics_SP}. Moreover, by the usual moment relations for exponential families, the posterior density $p(u|y_k) \in \mathcal{P}$ has moments:
\[
m'':= \left(\Lambda''\right)^{-1} \xi'',\quad \Sigma'':= \left(\Lambda''\right)^{-1}.
\]
Once again  by the exponential family identity $\mathbb{E}_\theta\left[T(u)\right]= \nabla\psi(\theta)$  with the Gaussian sufficient statistics being $T(u) = \left(u,-\frac{1}{2}uu^\top\right)$, the dual parameter at $\theta''$ is given by
\begin{align*}
    \bar{\eta}''(\theta'') = \mathbb{E}_{\theta''}\left[T(u)\right] = \left(m'', -\frac{1}{2}\left(\Sigma'' + (m'')(m'')^\top\right)\right) = \nabla\psi(\theta''),
\end{align*}
where $\psi$ is the same potential as defined in Equation~
\eqref{eqn:GP_potential}. Hence $\Phi_{\operatorname{upd}}$ also satisfies the fixed potential duality preservation condition of Definition~\ref{defn:Legendre_dynamics_SP}.

Therefore $\Phi_{\operatorname{upd}}$ and $\Phi_{\operatorname{pred}}$ both map $\Xi \rightarrow \Xi$ and preserve the Legendre duality $\bar{\eta} = \nabla\psi(\theta)$ with the same Gaussian potential $\psi$. Hence, by Equation~\eqref{eqn:g_k_GPR}, one GPR step (prediction + conditioning) given by $g_k\left(\cdot; y_k\right) = \Phi_{\operatorname{upd}}\left(\cdot; y_k\right) \circ \Phi_{\operatorname{pred}}$ is fixed potential discrete-time Legendre dynamics on the dualistic model $\left(\mathcal{P},\psi\right)$, as desired.
\end{proof}
In particular, one prediction-update step of an LTI-GPR can be seen as a fixed potential Legendre dynamics on the Gaussian exponential family with potential $\psi$ of Equation~\eqref{eqn:GP_potential}.
\begin{remark}
    A broad class of Gaussian process priors admit an LTI state space (hence Legendre dynamics) representation. In particular, if the kernel’s spectral density is a rational function of $\omega^2$, then the GP can be written as an LTI SDE  \cite{hartikainen2010kalman}. 
    
    This includes, for example, the Matern family with appropriate smoothness parameters (where the corresponding SDE is available in closed form). Whilst the squared-exponential kernel does not have an exact finite-dimensional LTI-SDE representation, finite order spectral or Taylor approximations yield approximate LTI state-space models and therefore approximate Legendre dynamics in the sense of the chosen truncation.

    In Appendix~\ref{app:OU_dynamics}, we further demonstrate that Ornstein-Uhlenbeck (OU) dynamics is a special case of Legendre dynamics.
\end{remark}

The Legendre dynamics framework thus captures the structural invariants of the input-side stochastic processes (such as LTI-GPR and OU). We next identify Legendre duality with Lagrangian graphs in $T^*Q$, and then characterize the Hamiltonian updates that preserve Legendre graphs. The characterization serves as the ``representational'' part of the main diagram. Concrete neural network realizations are then introduced in Section~\ref{sec:linearHamiltonian_src}

\section{Legendre duality as Lagrangian submanifolds}
\label{sec:duality_lagrangian}

This section and the next  will focus on the `representation' part of our main diagram:

\[
\{\text{input space},\text{structure}\}
\;\boxed{\xrightarrow{\ \text{representation}\ }}\;
\{\text{feature space},\text{structure}\}.
\]


In classical mechanics, symplectic maps arise as Hamiltonian flow maps and preserve the canonical two-form. In this section we identify the geometric object that carries Legendre duality and fix notation for the subsequent sections. For a $C^2$ potential $\psi \in C^{2}(Q)$ we define the Legendre graph:
\[
L_{\psi} := \{(q,p) \in T^{*}Q : p = d\psi(q)\}.
\]

We prove that $L_{\psi}$ is a Lagrangian submanifold of $(T^{*}Q,\omega)$. Thus, trajectories that preserve Legendre duality are precisely curves that remain inside a Lagrangian graph $L_{\psi}$. This geometric identification is the only ingredient needed in the next section, where we characterize which symplectomorphisms map such Lagrangian graphs to Lagrangian graphs and hence qualify as proper representation updates for Legendre dynamics.

We begin by recalling some elementary symplectic geometry.

Let $Q$ be an $n$-dimensional smooth manifold,  denoting the configuration manifold. Consider the cotangent bundle $M = T^* Q:=\left\{(q,p)\middle| q\in Q, \, p \in T_q^* Q\right\}$ which is also commonly referred to as the phase space. In local coordinates, a point in the phase space $M$ is given by $(q^1, \ldots, q^n, p_1, \ldots, p_n)$.

The cotangent bundle $T^*Q$ is endowed with the canonical one-form $\alpha$ also called the tautological one-form. In local coordinates, $\alpha$ is given by
\[
\alpha = p_i dq^i.
\]

The canonical symplectic form $\omega$ is a two-form on $T^*Q$ defined as the negative exterior derivative of the canonical one-form:
\[
\omega = -d\alpha.
\]
In local coordinates on $T^*Q$, it is given by:
\[
\omega = -d\left(p_idq^i\right)= -dp_i \wedge dq^i = dq^i\wedge dp_i.
\]

Hence the cotangent bundle phase space $M:=T^*Q$ forms a symplectic manifold $(M,\omega)$.

\begin{definition}
    Let $(M,\omega)$ be a $2n$-dimensional symplectic manifold. A submanifold $L\subset M$ is a \textbf{Lagrangian submanifold} if
    \begin{itemize}
        \item $\dim(L) = n = \frac{1}{2}\dim(M)$
        \item $\omega|_{L} \equiv 0$, i.e.
        \[
        \omega(u,v) = 0,\quad \forall u,v\in T_xL,\quad\forall x\in M.
        \]
    \end{itemize}
\end{definition}

The notion of `duality of Legendre type' is a relation between configuration space $Q$ and its cotangent fibres determined by a smooth potential function. More formally:

\begin{definition}[Duality of Legendre Type]
\label{defn:legendre_duality}
    Let $Q$ be an $n$-dimensional smooth manifold and let $\psi: Q \to \mathbb{R}$ be a smooth \textbf{potential} function. A point $(q, p)$ in the cotangent bundle $T^*Q$ is said to satisfy a \textbf{duality of Legendre type} generated by $\psi$ if its fiber coordinate $p$ is the exterior derivative of $\psi$ evaluated at its base coordinate $q$. In particular
    \[
    p = d\psi(q).
    \]

    In local coordinates $(q^1, \dots, q^n)$ on $Q$, this condition becomes $p_i = \frac{\partial\psi}{\partial q^i}$ for each $i=1, \dots, n$. The coordinates $q$ are often called the \textbf{primal coordinates}, while the coordinates $p$ are the \textbf{dual coordinates} in the context of this relationship.

    The set of all points in $T^*Q$ that satisfy this duality forms a set $L_\psi$, defined as the graph of the differential $d\psi$:
    \[
    L_\psi := \left\{(q,p) \in T^*Q\middle| p = d\psi(q)\right\}.
    \]
\end{definition}

\begin{remark}
Note that the geometric definition of Legendre graph $L_\psi$ (Definition~\ref{defn:legendre_duality}) requires $C^2$ potential $\psi$. Strict convexity is used in the statistical dualistic model interpretation of exponential families, where $\psi$ is the log-partition function and is strictly convex, provides the statistical motivation for studying Legendre dynamics.

The geometric results of this paper and their realizations (in Section~\ref{sec:linearHamiltonian_src}) operate at the level of $C^2$ potentials without requiring convexity.

\end{remark}

\begin{remark} The notion of duality of Legendre type arises in multiple fields. Indeed, the term Legendre was motivated by \textbf{Legendre transformation}, which was used in classical mechanics and information geometry to switch between different sets of state variables under a potential function. Our definition is a geometric generalization of this procedure.

    Comparison between classical mechanics and information geometry is summarized in Table~\ref{table:duality}.

    \begin{table}[h!]
    \centering
    \caption{Comparison of Legendre Type Duality in Different Fields}
    \label{table:duality}
    \begin{tabular}{lll}
    \toprule
    & \textbf{Classical Mechanics} & \textbf{Information Geometry} \\
    \midrule
    \textbf{Base Space} & Tangent Bundle $TQ$ & Statistical Manifold $Q$ (or $\Theta$) \\
    \textbf{Base Variable} & $(q, \dot{q})$ & Natural Parameter $\theta$ \\
    \textbf{Potential} & Lagrangian $L(q, \dot{q})$ & Cumulant Generator $\psi(\theta)$\\
    \textbf{Duality Definition} & $p = \nabla_{\dot{q}} L(q, \dot{q})$ & $\bar \eta = \nabla_{\theta} \psi(\theta)$ \\
    \textbf{Dual Variable} & Canonical Momentum $p$ & Expectation Parameter $ \bar  \eta$\\
    \textbf{Dual Space} & Cotangent Bundle $T^*Q$ & Cotangent Bundle $T^*Q$ \\
    \textbf{Dual Potential} & Hamiltonian $H(q, p) = p^T \dot{q} - L$ & Neg-entropy $\phi(\bar  \eta) = \theta^T \bar \eta - \psi$\\
    \textbf{Inverse Duality} & $\dot{q} = \nabla_p H(q, p)$ & $\theta = \nabla_{\bar \eta} \phi(\bar  \eta)$ \\
    \bottomrule
    \end{tabular}
    \end{table}
\end{remark}

Now we show that the set $L_\psi$ is a Lagrangian submanifold of $\left(T^*Q,\omega\right)$. In other words, the algebraic property of `duality of Legendre type' defined by potential $\psi$ can be realized geometrically by the corresponding Lagrangian submanifolds $L_\psi$.

\begin{theorem}
\label{thm:lagrangian_legendre_type}
    Let $Q$ be an $n$-dimensional smooth manifold and let $\psi: Q \to \mathbb{R}$ be a $C^2$ potential function. The set $L_\psi := \left\{(q,p) \in T^*Q\middle| p = d\psi(q)\right\} \subset T^* Q$ is a Lagrangian submanifold of $\left(T^*Q,\omega\right)$. 
\end{theorem}

\begin{proof}
    First notice that $\psi$ provides a natural parametrization of $L_\psi$ with coordinates given by $Q$ under the map:
    \begin{align*}
        \Phi_\psi:Q & \rightarrow T^*Q\\
         q &\rightarrow \left(q,d\psi(q)\right).
    \end{align*}
    Then $L_\psi = \operatorname{Im}\left(\Phi_\psi\right) = \Phi_\psi\left(Q\right)$. By construction $\Phi_\psi$ is smooth. We now show that $\Phi_\psi$ is an injective immersion:
    \begin{itemize}
        \item \textbf{Injectivity:} Let $q_1,q_2 \in Q$ be arbitrary. Suppose $\Phi_\psi(q_1) = \Phi_\psi(q_2)$, then $\left(q_1, d\psi(q_1)\right) = \left(q_2, d\psi(q_2)\right)$. Since the first components must equal, hence $q_1 = q_2$.
        \item \textbf{Immersion:} We need to show that $\operatorname{Jac}\Phi_\psi$ has rank $n$. In local coordinates of $T^*Q$, the map $\Phi_\psi$ is given by
        \[
        \Phi_\psi(q) = \left(q^1,\ldots,q^n,\frac{\partial\psi}{\partial q^1},\ldots,\frac{\partial \psi}{\partial q^n}\right).
        \]
        Hence the Jacobian $\operatorname{Jac}\Phi_\psi$ is a $2n\times n$ matrix given by
        \[
        \operatorname{Jac}\Phi_\psi = 
        \begin{bmatrix}
            \frac{\partial q^1}{\partial q^1} & \dots &\frac{\partial q^1}{\partial q^n} \\
            \vdots & & \vdots \\
            \frac{\partial q^n}{\partial q^1} & \dots &\frac{\partial q^n}{\partial q^n} \\ \cmidrule(lr){1-3}
            \frac{\partial p_1}{\partial q^1} & \dots &\frac{\partial p_1}{\partial q^n} \\
            \vdots & & \vdots \\
            \frac{\partial p_n}{\partial q^1} & \dots &\frac{\partial p_n}{\partial q^n}
        \end{bmatrix}
        = 
        \begin{bmatrix}
            I_n \\ 
            \operatorname{Hess}(\psi)\\
        \end{bmatrix}.
        \]
    \end{itemize}
    The top equality is due to the fact that $\frac{\partial q^i}{\partial q^j} = \delta^i_j$ on the same local chart. The bottom equality is by construction of the local coordinates $p_i= d\psi/\partial q^i$:
    \[
    \frac{\partial p_i}{\partial q^j} = \frac{\partial}{\partial q^j} \frac{\partial\psi}{\partial q^i} = \frac{\partial^2\psi}{\partial q^j\partial q^i}.
    \]
    Since $I_n$ has rank $n$, so does $\operatorname{Jac}\Phi_\psi$ and therefore $\Phi_\psi$ is an immersion.

    Notice that $\pi\circ \Phi_\psi = \operatorname{Id}_Q$ (projection to the base), $\Phi_\psi$ is a homeomorphism onto its image. Since $\Phi_\psi$ is a smooth injective immersion, it is an embedding. Hence $L_\psi$ is an embedded submanifold of $T^* Q$. Moreover, since $Q$ is a smooth $n$-dimensional manifold, the dimension condition of a Lagrangian submanifold is satisfied as follows:
    \[
    \dim\left(L_\psi\right) = n = \frac{1}{2}\dim\left(T^*Q\right).
    \]

    Now it remains to show that $L_\psi$ is isotropic with respect to $\omega$, i.e. $\omega|_{L_\psi} \equiv 0$. Since $\Phi_\psi$ is an injective immersion, it suffices to show that $\Phi_\psi^*\omega \equiv 0$. The reason is, since $\Phi_\psi$ is an injective immersion: for each $x := \Phi_\psi(q) \in L_\psi$ and each $v\in T_x L_\psi$, there exists a unique $u\in T_q Q$ such that $v = \left(\Phi_\psi\right)_*u$. Thus, if $\left(\Phi^*_\psi\omega\right)_q\left(u_1,u_2\right) = 0$ for all $u_1,u_2\in T_q Q$, then for all $x = \Phi_\psi(q) \in L_\psi$ and all $v_1,v_2\in T_x L_\psi$, we have:
    \[
    \omega_x(v_1,v_2) = \omega_{\Phi_\psi(q)}\left(\left(\Phi_\psi\right)_* u_1, \left(\Phi_\psi\right)_* u_2\right) = \left(\Phi_\psi^*\omega\right)_q\left(u_1,u_2\right)= 0.
    \]
    
    
    The above implies $\omega|_{L_\psi} \equiv  0$. So now it remains to show that $\Phi_\psi^*\omega \equiv 0$. We begin by considering the tautological one-form $\alpha = p_i dq^i$. Recall by construction that
    \[
    p_i = \frac{\partial\psi}{\partial q^i}.
    \]

    The pullback $\Phi_\psi^*\alpha$ is a one-form on $Q$:
    \[
    \Phi_\psi^* \alpha = \frac{\partial \psi }{\partial q^i} dq^i = d\psi,
    \]
    where the last equality is by definition of exterior derivative. Next we consider the pullback of $\omega$ under $\Phi_\psi$. First recall $\omega = -d\alpha$ by definition. Since pullback commutes with exterior derivative, we have $d\circ\Phi_\psi^* = \Phi_\psi^* \circ d$ and hence
    \begin{align*}
    \Phi_\psi^*\omega &= \Phi^*_\psi\left(-d\alpha\right) = -d\left(\Phi_\psi^*\alpha\right) \\
    &  = -d(d\psi) = 0,
    \end{align*}
    where the last equality is due to the fact that $d^2\psi = 0$ for smooth function $\psi$. This means $\Phi_\psi^*\omega = 0$, which implies $\omega|_{L_\psi}\equiv 0$ and the proof is complete.
\end{proof}

We call $L_\psi$ the \textbf{Lagrangian submanifold of Legendre type corresponding to the potential $\psi$}. Thus Legendre duality on the parameter space $Q$ can be realized geometrically as a Lagrangian submanifold $L_\psi\subset T^*Q$. In the next sections we study how symplectomorphisms, in particular the symplectic reservoir maps, act on such Lagrangian graphs and the sufficient conditions for which Legendre structure is preserved.

\section{Legendre preserving symplectomorphisms}
\label{sec:preservation}

In this section we continue to investigate the representation part of our main diagram:

\[
\{\text{input space},\text{structure}\}
\;\boxed{\xrightarrow{\ \text{representation}\ }}\;
\{\text{feature space},\text{structure}\}
\]

Legendre dynamics are trajectories across Lagrangian graphs $L_{\psi}$. Under the geometrical characterization of Legendre graphs as Lagrangian submanifolds described in Section~\ref{sec:duality_lagrangian}, we now characterize symplectomorphisms that preserve Lagrangian submanifolds. We prove that a symplectomorphism preserves Legendre type structure if and only if it has the normal form $F_t = \tau_{d\chi_t} \circ f_t^{\sharp}$. That is, an exact fiber translation by $d\chi_t$ composed with a cotangent lift of a base diffeomorphism $f_t:Q \to Q$. This completes the structural classification of admissible representation updates that preserve Legendre dynamics. This characterization provides the design space from which SR will construct its representation.

Preservation of Lagrangian submanifolds is a broader symplectic geometric notion. The present paper focuses on the sharper condition needed for Legendre dynamics: preservation of Lagrangian submanifolds of Legendre type generated by potentials. This is the crucial property for applications like information geometry where the Legendre duality is paramount:

\begin{definition}
    Let $Q$ be an $n$-dimensional smooth manifold and let $F:T^* Q \rightarrow T^*Q$ be a diffeomorphism on the cotangent bundle $T^*Q$. We say that $F$ \textbf{preserves the Legendre type structure} if for any smooth potential function $\psi:Q\rightarrow \mathbb{R} $, the map $F$ transforms the Lagrangian submanifold of Legendre type $L_\psi$ into another Lagrangian submanifold of Legendre type. That is, there exists a potential function $\psi_{out}:Q\rightarrow \mathbb{R}$ such that the image $F(L_\psi) = L_{\psi_{out}}$.
\end{definition}





For the rest of the discussion, we denote the zero section by:
\[
\Gamma_0:=\left\{ (q,p) \in T^*Q \middle| p = 0\right\},
\]
where $\Gamma_0$ is the graph of the closed $1$-form $\alpha = 0$.

 {
\begin{lemma}[Zero-section normal form]\label{lem:zero-section}
Let $\Psi : T^*Q \to T^*Q$ be a symplectomorphism which:
\begin{enumerate}[label=(\roman*)]
    \item\label{hyp:zero}[Preserves the zero section]: $\Psi(\Gamma_0) = \Gamma_0$, where
          $\Gamma_0 \subset T^*Q$ is the zero section;
    \item\label{hyp:germ}[Sends every germ of an exact graph to a local graph]:  for every $z=(q,p)\in T^*Q$ and every
    germ of a smooth exact $1$-form $d\varphi$ at $q$ with
    $d\varphi(q)=p$, the image of the graph
    \[
        \Gamma_{d\varphi}:=\{(x,d\varphi(x))\}
    \]
    under $\Psi$ is, near $\Psi(z)$, the graph of a smooth local
    $1$-form over an open subset of $Q$.
\end{enumerate}
Then there exists a diffeomorphism $f : Q \to Q$ such that
$\Psi = f^\sharp$, where
\[
    f^\sharp(q, p) = \left(f(q),\, (f^{-1})^* p\right).
\]
\end{lemma}

\begin{proof}
Throughout this proof we use ``germ'' in the standard sense: two local
sections of $\pi:T^*Q\to Q$ determine the same germ at $q\in Q$ if
they agree on some open neighbourhood of $q$.

Hypothesis~\ref{hyp:germ} thus means the following: For each
$z=(q,p)\in T^*Q$ and each local potential $\varphi$ with
$d\varphi(q)=p$, there exists an open neighbourhood $U$ of
$\Psi(z)$ such that $\Psi(\Gamma_{d\varphi})\cap U$ is the graph of a smooth local section of $T^*Q\to Q$ over an open subset of $Q$.

To show that $\Psi$ is a cotangent lift, we first show that $\Psi$ is fibre-preserving. In particular, its base component depends only on $q$, not on the covector $p$ and we show that $\pi\circ\Psi$ is constant along each cotangent fibre. In local canonical
coordinates near an arbitrary point $z=(q_0,p_0)$, we write
\[
    \Psi(q,p)=\bigl(\mathcal Q(q,p),\mathcal P(q,p)\bigr),
\]
where $\mathcal Q$ is the base component and $\mathcal P$ is the
fibre component of $\Psi$ respectively. Set
\[
    A=\frac{\partial \mathcal{Q}}{\partial q} (z) = D_q\mathcal Q(z),
    \qquad
    B=\frac{\partial \mathcal{Q}}{\partial p} (z) = D_p\mathcal Q(z).
\]
Thus by construction
    $A:T_{q_0}Q\to T_{\mathcal Q(z)}Q$ and $
    B:T^*_{q_0}Q\to T_{\mathcal Q(z)}Q .$

Let $S:T_{q_0}Q\to T^*_{q_0}Q$ be an arbitrary symmetric linear map. Then in local coordinates, there exists a smooth local potential $\varphi$ such that $S$ is the Hessian of $\varphi$, in particular
\[
    d\varphi(q_0)=p_0,
    \qquad
    \operatorname{Hess}_{q_0}\varphi=S.
\]
By construction, the tangent space to the graph $\Gamma_{d\varphi}$ at $z$ is
\[
    T_z\Gamma_{d\varphi}
    =
    \{(\delta q,S\delta q):\delta q\in T_{q_0}Q\}.
\]
By hypothesis~\ref{hyp:germ}, $\Psi(\Gamma_{d\varphi})$ is locally a graph of a smooth local $1$-form over $Q$ near $\Psi(z)$. A smooth local $1$-form over an open subset $V\subset Q$ is exactly a smooth section $\sigma:V\rightarrow T^*Q$ of the cotangent projection $\pi:T^*Q \rightarrow Q$. Hence equivalently, after shrinking neighbourhoods if
necessary, there exist an open set $V\subset Q$ and a smooth local section
$\sigma:V\to T^*Q$ of $\pi:T^*Q\to Q$ such that
$\Psi(\Gamma_{d\varphi})\cap U$ agrees with $\sigma(V)$ near $\Psi(z)$.
Since $\sigma$ is a section, $\pi\circ\sigma=\mathrm{id}_V$, so
\[
    \pi|_{\sigma(V)}:\sigma(V)\to V
\]
is a local diffeomorphism with inverse $\sigma$. Therefore its derivative
at $\Psi(z)$ is an isomorphism. Equivalently by chain rule, the map
\[
    D\pi_{\Psi(z)}\circ D\Psi_z:
    T_z\Gamma_{d\varphi}\to T_{\pi(\Psi(z))}Q
\]
is an isomorphism. In local coordinates, this projected derivative is thus
\[
    \delta q
    \longmapsto
    D_q\mathcal Q(z)\delta q+D_p\mathcal Q(z)S\delta q
    =:
    (A+BS)\delta q.
\]
Hence $A+BS$ is invertible for every symmetric $S$. Setting $S=0$ shows that $A$ itself is invertible. We now show that $B$ must vanish. Suppose by contrary that $B\neq0$. Then there exists $\vartheta\in T^*_{q_0}Q$ such that
\[
    B\vartheta\neq0.
\]
Let $v=-A^{-1}B\vartheta$, then since $B\vartheta\neq0$ and $A$ is invertible, we have $v\neq0$.

Choose a symmetric linear map
\[
    S:T_{q_0}Q\to T^*_{q_0}Q
\]
such that
\[
    Sv=\vartheta.
\]
Note that one can construct $S$ by choosing a covector $\ell\in T^*_{q_0}Q$ with $\ell(v)=1$, and then define a symmetric bilinear form 
\[
    b(x,y)
    =
    \ell(x)\vartheta(y)+\ell(y)\vartheta(x)-\vartheta(v)\ell(x)\ell(y).
\]
Then $b(v,y)=\vartheta(y)$ for all $y$ and we can set $S$ to be the symmetric linear map corresponding to $b$, i.e. $Sx(y)=b(x,y)$. Then
$Sv=\vartheta$.

By construction of $S$, we have
\[
    (A+BS)v
    =
    Av+B\vartheta
    =
    -B\vartheta+B\vartheta
    =
    0,
\]
this contradicts the fact that $A+BS$ is invertible, hence we must have $B=0$.

Since $z$ was arbitrary, $B \equiv 0 $ means
\[
    D_p\mathcal Q\equiv0.
\]

In particular the base component $\mathcal Q(q,p)$ is independent of $p$, and there exists smooth $f:Q\to Q$ such that
\[
    \mathcal Q(q,p)=f(q),
\]
equivalently we have
\[
    \pi\circ\Psi=f\circ\pi.
\]
This means that $\Psi$ maps each fibre $T_q^*Q$ into the fibre
$T^*_{f(q)}Q$, i.e. $\Psi$ is fibre-preserving and: 
\[
    \Psi(q,p)=\bigl(f(q),\mathcal P(q,p)\bigr).
\]

We next show that $f$ is a diffeomorphism. Since $\Psi$ is a symplectomorphism, it is a diffeomorphism. by hypothesis~\ref{hyp:zero}, we have $\Psi(\Gamma_0)=\Gamma_0$, so the restriction 
\[
    \Psi|_{\Gamma_0}:\Gamma_0\to\Gamma_0
\]
is a diffeomorphism.  Under the identification $Q\cong\Gamma_0$, $q\mapsto(q,0)$, this restriction is exactly the map $q\mapsto f(q)$, since
\[
    \Psi(q,0)=(f(q),0).
\]
Therefore $f:Q\to Q$ is itself a diffeomorphism. Now it remains to identify the fibre component $\mathcal P$ in $\Psi(q,p)=\left(f(q),\mathcal P(q,p)\right)$. We first compute the pullback of the canonical symplectic form $\omega$ under $\Psi$. In output canonical Darboux coordinates $(\hat q,\hat p)$ on $T^*Q$,
the symplectic form is
\[
    \omega=d\hat q^a\wedge d\hat p_a.
\]
where $\hat{q}^a=f^a(q)$ and $\hat{p}_a=\mathcal P_a(q,p)$ under $\Psi$.

Therefore
\[
    \Psi^*\omega
    =
    df^a(q)\wedge d\mathcal P_a(q,p).
\]
Expanding both parts gives
\[
    df^a
    =
    \frac{\partial f^a}{\partial q^i}\,dq^i, \qquad d\mathcal P_a
    =
    \frac{\partial \mathcal P_a}{\partial q^j}\,dq^j
    +
    \frac{\partial \mathcal P_a}{\partial p_j}\,dp_j.
\]
The $dq\wedge dp$ part of $\Psi^*\omega$ is thus given by
\[
    \frac{\partial f^a}{\partial q^i}
    \frac{\partial \mathcal P_a}{\partial p_j}
    \,dq^i\wedge dp_j.
\]
Since $\Psi$ is a symplectomorphism, we have,
\[
    \Psi^*\omega=\omega=dq^i\wedge dp_i.
\]
Comparing the $dq\wedge dp$ coefficients gives
\[
    \sum_a
    \frac{\partial f^a}{\partial q^i}
    \frac{\partial \mathcal P_a}{\partial p_j}
    =
    \delta_{ij}.
\]
Equivalently, notice in matrix form,
\[
    Df(q)^\top
    \frac{\partial \mathcal P}{\partial p}(q,p) =I \implies \frac{\partial \mathcal P}{\partial p}(q,p)
    =
    Df(q)^{-\top}.
\]
In particular, this means $\partial\mathcal P/\partial p$ is independent of $p$, and hence there exists a $1$-form $\gamma$ on $Q$ such that
\[
    \mathcal P(q,p)=Df(q)^{-\top}p+\gamma(q).
\]

Finally, again by hypothesis~\ref{hyp:zero} we have
\[
    \Psi(q,0)=(f(q),0).
\]
Hence substituting $p=0$ into
\[
    \mathcal P(q,p)=Df(q)^{-\top}p+\gamma(q)
\]
gives
\[
    \mathcal P(q,0)=\gamma(q).
\]
But $\Psi(q,0)=(f(q),0)$, so $\mathcal P(q,0)=0$. Hence
\[
    \gamma(q)=0.
\]
Therefore
\[
    \mathcal P(q,p)=Df(q)^{-\top}p.
\]
Since $Df(q)^{-\top}p=(f^{-1})^*p$, the pullback of $p\in T_q^*Q$ by $f^{-1}$,  we obtain the desired equality:
\[
    \Psi(q,p)
    =
    \bigl(f(q),Df(q)^{-\top}p\bigr)
    =
    \bigl(f(q),(f^{-1})^*p\bigr)
    =
    f^\sharp(q,p).
\]
\end{proof}


\begin{theorem}[Closed-graph normal form]
\label{thm:graph_liftcomposition}
Let $Q$ be a smooth manifold and let $F : T^*Q \to T^*Q$ be a
symplectomorphism.
Then $F$ maps the graph of every closed $1$-form on $Q$ to the graph of
another closed $1$-form if and only if
\[
    F = \tau_\beta \circ f^\sharp,
\]
where $f : Q \to Q$ is a diffeomorphism, $\beta$ is a closed $1$-form on
$Q$, $\tau_\beta(q, p) = (q,\, p + \beta(q))$, and
$f^\sharp(q, p) = \bigl(f(q),\, (f^{-1})^* p\bigr)$.
\end{theorem}

\begin{proof}
\textbf{($\Leftarrow$)} \enspace Assume $F$ can be written as the composition of a cotangent lift of a diffeomorphism $f:Q\rightarrow Q$ and fibre-wise aﬃne translation of a closed $1$-form $\beta$. For every closed $1$-form $\alpha$ on $Q$, by construction we have
\[
f^\sharp(\Gamma_\alpha)=\Gamma_{(f^{-1})^*\alpha}.
\]
Since $d\alpha=0$, we have
\[
d((f^{-1})^*\alpha)=(f^{-1})^*(d\alpha)=0.
\]
Hence $(f^{-1})^*\alpha$ is closed. Applying the fibre translation gives
\[
F(\Gamma_\alpha)
=
\tau_\beta(\Gamma_{(f^{-1})^*\alpha})
=
\Gamma_{(f^{-1})^*\alpha+\beta}.
\]
$(f^{-1})^*\alpha+\beta$ is closed since both $(f^{-1})^*\alpha$ and $\beta$ are closed. Hence $F$ maps graphs of
closed $1$-forms to graphs of closed $1$-forms.

\textbf{($\Rightarrow$)}\enspace Conversely suppose $F$ maps the graph of every closed $1$-form on Q to the graph of another closed $1$-form on $Q$. The zero section $\Gamma_0$ is the graph of the closed $1$-form
$\alpha = 0$, so there exists closed $1$-form $\beta$ such that 
\[
F(\Gamma_0) = \Gamma_\beta.
\]
Now the goal is to show that $F$ is a composition of a cotangent lift and an affine translation. It suffices to show that the following map is a cotangent lift:
\[
\Psi:= \tau_{-\beta}\circ F,
\]
where $\tau_{-\beta}$ is the fibre-wise affine translation by $-\beta$. 
Since $d\beta=0$, the fibre translation $\tau_{-\beta}$ is a
symplectomorphism. Hence $\Psi$ is also a symplectomorphism. It remains to verify that $\Psi$ satisfies the two hypotheses of Lemma~\ref{lem:zero-section}. By construction $\Psi(\Gamma_0) = \Gamma_0$, so hypothesis~\ref{hyp:zero} of Lemma~\ref{lem:zero-section} holds. 

It remains to show that $\Psi$ satisfies hypothesis~\ref{hyp:germ} of Lemma~\ref{lem:zero-section}.
Let $z = (q, p) \in T^*Q$ and let $\varphi$ be a local smooth potential
near $q$ with $d\varphi(q) = p$. Shrinking the domain of $\varphi$ if necessary, choose a bump function supported in this domain that equals to one near $q$. Multiplying $\varphi$ by this bump function and extending by zero gives a global
smooth function $\varphi_e:Q\to\mathbb R$ such that
\[
d\varphi_e=d\varphi
\]
near $q$. Since $d\varphi_e$ is exact, it is closed. By assumption, there exists
a closed $1$-form $\alpha_e'$ such that
\[
F(\Gamma_{d\varphi_e})=\Gamma_{\alpha_e'}.
\]
Since $\beta$ is closed,
\[
\Psi(\Gamma_{d\varphi_e})
=
\tau_{-\beta}(\Gamma_{\alpha_e'})
=
\Gamma_{\alpha_e'-\beta},
\]
which is again the graph of a closed $1$-form, hence in particular a
local graph over $Q$.

Since $d\varphi$ and $d\varphi_e$ have the same germ at $q$, their graphs
agree near~$z$, so $\Psi(\Gamma_{d\varphi})$ is locally a graph
over~$Q$ near $\Psi(z)$ and the hypothesis~\ref{hyp:germ} of Lemma~\ref{lem:zero-section} holds.

Therefore by Lemma~\ref{lem:zero-section}, $\Psi = f^\sharp$ for some
diffeomorphism $f : Q \to Q$, and hence
$F = \tau_\beta \circ f^\sharp$.

\end{proof}

\begin{theorem}[Legendre-graph normal form]
\label{thm:preserve_legendre}
Let $Q$ be an $n$-dimensional smooth manifold and let $F:T^*Q\to T^*Q$ be a symplectomorphism. Then $F$ preserves Legendre type structure if and only if
\[
    F=\tau_{d\chi}\circ f^\sharp
\]
for some diffeomorphism $f:Q\to Q$ and some smooth function
$\chi:Q\to\mathbb R$. Here
\[
    \tau_{d\chi}(q,p)=(q,p+d\chi(q)),
\]
and
\[
    f^\sharp(q,p)=\left(f(q),(f^{-1})^*p\right).
\]
Moreover, in this case
\[
    F(L_\psi)=L_{\psi\circ f^{-1}+\chi}.
\]
\end{theorem}

\begin{proof}
\textbf{($\Leftarrow$)} Suppose$ F=\tau_{d\chi}\circ f^\sharp$.
Let
\[
    L_\psi=\{(q,p)\in T^*Q:p=d\psi(q)\}
\]
be a Legendre graph. If $(q,p)\in L_\psi$, then $p=d\psi(q)$. By definition of  cotangent lift,
\[
    f^\sharp(q,d\psi(q))
    =
    \left(f(q),(f^{-1})^*d\psi\right).
\]
Since pullback commutes with exterior differentiation,
\[
    (f^{-1})^*d\psi=d(\psi\circ f^{-1}).
\]
Let $q'=f(q)$, we get
\[
    f^\sharp(q,d\psi(q))
    =
    \left(q',d(\psi\circ f^{-1})(q')\right).
\]
Now apply the  fibre translation under $d\chi$, we have
\[
    \tau_{d\chi}
    \left(q',d(\psi\circ f^{-1})(q')\right)
    =
    \left(q',d(\psi\circ f^{-1}+\chi)(q')\right).
\]
Hence
\[
    F(L_\psi)=L_{\psi\circ f^{-1}+\chi}.
\]
Therefore $F$ preserves Legendre type structure.

\textbf{($\Rightarrow$)} Conversely, suppose $F$ preserves Legendre
type structure. Consider the  zero potential given by $\psi=0$, since
\[
    L_0=\Gamma_0,
\]
by assumption of Legendre preservation of $F$, there exists a globally defined smooth function
$\chi:Q\to\mathbb R$ such that
\[
    F(\Gamma_0)=F(L_0)=L_\chi=\Gamma_{d\chi}.
\]
Set
\[
    \Psi:=\tau_{-d\chi}\circ F.
\]
Since $\tau_{-d\chi}$ is an exact fibre translation, it is a
symplectomorphism. Hence $\Psi$ is also a symplectomorphism. Moreover by construction, we have
\[
    \Psi(\Gamma_0)
    =
    \tau_{-d\chi}(F(\Gamma_0))
    =
    \tau_{-d\chi}(\Gamma_{d\chi})
    =
    \Gamma_0.
\]
Thus hypothesis~\ref{hyp:zero} of Lemma~\ref{lem:zero-section} holds.

We now verify that hypothesis~\ref{hyp:germ} of Lemma~\ref{lem:zero-section} holds. Let $z=(q,p)\in T^*Q$, and let $\varphi$ be a local smooth potential near $q$ such that $d\varphi(q)=p$. Shrinking the domain of $\varphi$ if necessary, choose a bump
function supported in this domain and equal to one near $q$. Multiplying
$\varphi$ by this bump function and extending by zero gives a global
smooth function $\varphi_e:Q\to\mathbb R$
such that
\[
    d\varphi_e=d\varphi
\]
near $q$. Therefore $\Gamma_{d\varphi_e}$ and $\Gamma_{d\varphi}$
agree as germs near $z$.

Since $F$ preserves Legendre type structure, there exists a smooth
potential $\psi_{\mathrm{out}}:Q\to\mathbb R$ such that
\[
    F(L_{\varphi_e})=L_{\psi_{\mathrm{out}}}.
\]
Now applying $\tau_{-d\chi}$, we obtain
\[
    \Psi(L_{\varphi_e})
    =
    \tau_{-d\chi}(L_{\psi_{\mathrm{out}}})
    =
    L_{\psi_{\mathrm{out}}-\chi}.
\]
Since $\Gamma_{d\varphi_e}$ and $\Gamma_{d\varphi}$ agree as germs near $z$, thus $\Psi$ maps the germ of $\Gamma_{d\varphi}$ near $z$ to the
germ of a local graph near $\Psi(z)$. Hence
hypothesis~\ref{hyp:germ} of Lemma~\ref{lem:zero-section} holds.

By Lemma~\ref{lem:zero-section}, there exists a diffeomorphism
$f:Q\to Q$ such that
\[
    \Psi=f^\sharp.
\]
Therefore
\[
    F=\tau_{d\chi}\circ\Psi
    =
    \tau_{d\chi}\circ f^\sharp.
\]
In particular, we have the desired equality
\[
F(L_\psi)=L_{\psi\circ f^{-1}+\chi}.
\]
\end{proof}

\begin{remark}[No global exactness assumption]
The proof of Theorem~\ref{thm:preserve_legendre} does not require an assumption such as
\[
H^1_{\mathrm{dR}}(Q)=0.
\]
The argument uses only the extension of local potentials by bump functions. Exactness of the fibre translation is forced by applying $F$ to the zero
section:
\[
F(L_0)=L_\chi=\Gamma_{d\chi}.
\]
Thus the translation term is globally exact because preservation of Legendre
type structure is applied to the global potential $\psi=0$.
\end{remark}

\begin{remark}[General versus fixed-potential preservation]
The preservation theorem gives general Legendre-graph preservation:
\[
    L_\psi \mapsto L_{\psi\circ f^{-1}+\chi}.
\]
Fixed-potential preservation is the special case in which
\[
    \psi\circ f^{-1}+\chi=\psi+c
\]
for some constant $c$.
\end{remark}

\begin{remark}[Preservation on state and readout]
\label{rmk:readout}
The preservation statement is made at the level of the cotangent-bundle state representation $(q_k,p_k)\in T^*Q$. Downstream task readouts may be trained on top of this representation, but they are not part of the graph-preservation claim.
\end{remark}

}

 {

\section{Hamiltonian realizations of Legendre-preserving maps}
\label{sec:linearHamiltonian_src}

This section closes the loop. Using the characterization from Section~\ref{sec:preservation}, we introduce the $\left\{\text{feature space}, \text{structure}\right\}$ part of the diagram.
\[
\left\{\text{input space}, \text{structure}\right\} \xrightarrow[\text{}]{\text{representation}} 
\boxed{\left\{\text{feature space}, \text{structure}\right\}}.
\]

In particular, we realize the Legendre-preserving normal form at the neural network level. We first separate generic symplecticity from Legendre preservation, then construct linear and nonlinear Hamiltonian SR updates that remain inside the Legendre preserving class. This establishes Hamiltonian SR as an invariance-preserving representation for Legendre dynamics.

\begin{theorem}
\label{thm:linear_hamiltonian_specialform}
    Let $A$ be a complete vector field on $Q$ and let $V:Q \rightarrow \mathbb{R}$ be smooth. Let $H(q,p) = \langle p,A(q)\rangle + V(q)$ be a Hamiltonian on $T^*Q$ that is at most linear in the momentum $p$. For each $t$, the flow map $\phi_t : T^*Q \rightarrow T^*Q$ generated by $H$ is a Legendre preserving symplectomorphism of the form $\phi_t = \tau_{d\chi_t}\circ f^\sharp_t$.
\end{theorem}

\begin{proof}
    A Hamiltonian that is at most linear in $p$ can be written as
    \begin{align}
    \label{eqn:linear_p_Hamiltonian}
        H(q,p) = \langle p,A(q)\rangle + V(q),
    \end{align}
    where $A$ is a complete vector field on $Q$ and $V:Q\rightarrow \mathbb{R}$ is a smooth function. 

    In local Darboux coordinates $(q^i, p_i)$ on $T^*Q$,  we can write Equation~\eqref{eqn:linear_p_Hamiltonian} as
    \[
    H(q,p) = p_iA^i(q) + V(q).
    \]
    The flow $\phi_t$ of $H$ is determined by the Hamiltonian equations, broken down into two parts:
    \begin{align}
    \label{eqn:dot_q}
        \dot{q}^i = \frac{\partial H}{\partial p_i} &= \frac{\partial}{\partial p_i} \left(p_jA^j(q) + V(q)\right) = A^i(q),
    \end{align}
    \begin{align}
    \label{eqn:dot_p}
        \dot{p}_i = -\frac{\partial H}{\partial q_i} &= -\frac{\partial}{\partial q_i} \left(p_jA^j(q) + V(q)\right) = -p_j\frac{\partial A^j}{\partial q^i} - \frac{\partial V}{\partial q^i}.
    \end{align}

    From Equation~\eqref{eqn:dot_q}, we define $f_t:Q\rightarrow Q$ to be the flow generated by complete vector field $A$. Hence $f_t$ is a diffeomorphism on $Q$ and for all initial point $q_0\in Q$, the curve $q(t) = f_t(q_0)$ is the unique solution to the differential equation $\dot{q} = A(q)$ with the initial condition $q(0) = q_0$.

    Fix $q_0 \in Q$, let $\tilde{p}(t)$ be the curve of covectors in $T^*Q$ obtained by pulling $p(t)$ back along $f_t$ from $q(t) = f_t(q_0)$ to $q_0$. In particular, we define the pullback covector $\tilde{p}(t) := \left(f_t\right)^*p(t)$ and more specifically 
    \begin{align}
        \label{eqn:p_tilde}
        \tilde{p}_i(t) := \sum_j p_j(t)\frac{\partial q^j(t)}{\partial q_0^i} .
    \end{align}
    In these two lines the summation indices are summed explicitly for clarity and the Einstein summation convention is omitted. By construction,
    \begin{align}
    \label{eqn:p_j_t}
        p_j(t) &=p_j(t) \cdot \operatorname{Jac}\left(f_t\right)\cdot \operatorname{Jac}\left(f_t^{-1}\right) =  \sum_k p_k(t)\cdot \frac{\partial q^k(t)}{\partial q_0^i}\cdot \frac{\partial q_0^i}{\partial q^j(t)} \nonumber \\
        &= \sum_i \tilde{p}_i(t)\cdot \frac{\partial q_0^i}{\partial q^j(t)}.
    \end{align}
    Evaluating at $q_0$ and omitting the dependency on $t$ for simplicity, we get
    \begin{align}
    \label{eqn:dt_tilde_p_i}
        \frac{\partial }{\partial t} \tilde{p}_i &= \dot{p}_j \cdot \frac{\partial q^j}{\partial q_0^i} + p_j \frac{\partial }{\partial t} \frac{\partial q^j }{\partial q_0^i} \nonumber\\
        &= \left[ -p_k\frac{\partial A^k}{\partial q^j} - \frac{\partial V}{\partial q^j}\right]\cdot \frac{\partial q^j}{\partial q_0^i} + p_j \frac{\partial }{\partial t} \frac{\partial q^j }{\partial q_0^i}, 
    \end{align}
     where the last equality is due to Equation~\eqref{eqn:dot_p}. Moreover, by Equation~\eqref{eqn:dot_q} since $\dot{q}^j = A^j(q)$ the second part of Equation~\eqref{eqn:dt_tilde_p_i} can be expanded as
    \begin{align*}
        \frac{\partial }{\partial t} \frac{\partial q^j }{\partial q_0^i} = \frac{\partial}{\partial q_0^i}\left(\dot{q}^j\right) = \frac{\partial}{\partial q_0^i}\left(A^j(q)\right) = \frac{\partial A^j}{\partial q^k}\frac{\partial q^k}{\partial q_0^i}.
    \end{align*}
    Hence we get:
    \begin{align*}
        \frac{\partial }{\partial t} \tilde{p}_i 
        &= \left[ -p_k\frac{\partial A^k}{\partial q^j} - \frac{\partial V}{\partial q^j}\right]\cdot \frac{\partial q^j}{\partial q_0^i} + p_j \left[\frac{\partial A^j}{\partial q^k}\frac{\partial q^k}{\partial q_0^i}\right].
    \end{align*}
    By symmetry we can swap the arbitrary dummy indices $j,k$ of the Einstein summation convention in the first part, and get
    \begin{align*}
        \frac{\partial }{\partial t} \tilde{p}_i 
        &=  -p_j\left[\frac{\partial A^j}{\partial q^k}\cdot \frac{\partial q^k}{\partial q_0^i}\right] - \frac{\partial V}{\partial q^k}\cdot \frac{\partial q^k}{\partial q_0^i} + p_j \left[\frac{\partial A^j}{\partial q^k}\frac{\partial q^k}{\partial q_0^i}\right] \\
        &= - \frac{\partial V}{\partial q^k}\cdot \frac{\partial q^k}{\partial q_0^i} = -\frac{\partial }{\partial q_0^i}\left[V(q(t))\right].
    \end{align*}
    Now by the Fundamental Theorem of Calculus,
    \begin{align}
    \label{eqn:tilde_p_int_1}
        \tilde{p}_i(t) - \tilde{p}_i(0) &= -\int_0^t \frac{\partial }{\partial q_0^i}V(q(\tau))d\tau \nonumber \\ 
        &= -\frac{\partial }{\partial q_0^i}\left[\int_0^t V(q(\tau))d\tau \right].
    \end{align}

    At $t = 0$, $q(0) = f_0(q_0) = q_0$, hence by Equation~\eqref{eqn:p_tilde}, $\tilde{p}_i(0) = p_j(0)\cdot \frac{\partial q^j_0}{\partial q_0^i} = p_j(0)\cdot\delta^j_i = (p_0)_i$. Therefore Equation~\eqref{eqn:tilde_p_int_1} becomes
    \begin{align*}
    \tilde{p}_i(t)= (p_0)_i-\frac{\partial }{\partial q_0^i}\left[\int_0^t V(q(\tau))d\tau\right] .
    \end{align*} 

    By Equation~\eqref{eqn:p_j_t}, hence
    \begin{align*}
        p_j(t) &= \tilde{p}_i(t)\cdot \frac{\partial q_0^i}{\partial q^j(t)} \nonumber\\
        &= (p_0)_i\cdot \frac{\partial q_0^i}{\partial q^j(t)}  -\frac{\partial }{\partial q_0^i}\left[\int_0^t V(q(\tau))d\tau\right] \frac{\partial q_0^i}{\partial q^j(t)} \nonumber .
        \end{align*}
    Since $\frac{\partial q_0^i}{\partial q^j(t)} = \operatorname{Jac}(f_t^{-1})_{ij}$, the above expression becomes:
    \begin{align}
    \label{eqn:tilde_p_int_3}
        p_j(t)& = \left((f_t^{-1})^* p_0\right)_j - \left(\frac{\partial }{\partial q_0^i} \left[\int_0^t V(q(\tau))d\tau\right]\right) \frac{\partial q_0^i}{\partial q^j(t)}\nonumber \\
        & =\left((f_t^{-1})^* p_0\right)_j -\frac{\partial}{\partial q^j(t)} \left[\int_0^t V(q(\tau))d\tau\right],
    \end{align}
    where we recall $f_t^\sharp(q_0,p_0) = \left(f_t(q_0),(f_t^{-1})^*p_0\right)$. In particular, $\left((f_t^{-1})^* p_0\right)_j$ is the $j^{th}$ coordinate of the momentum part of $f_t^\sharp$ acting on $(q_0, p_0)$.
    
    Now define $\chi_t(q)$ to be the following function: 
    \begin{align}
        \label{eqn:chi_t}
        \chi_t(q) := - \int_0^t V\left(f_{\tau-t}(q)\right)d\tau,
    \end{align}
    then 
    \begin{align*}
        - \frac{\partial}{\partial q^j(t)} \left[\int_0^t V(q(\tau))d\tau\right] = \frac{\partial \chi_t}{\partial q^j}\left(q(t)\right).
    \end{align*}

    Therefore in vector form, Equation~\ref{eqn:tilde_p_int_3} becomes
    \begin{align}
        \label{eqn:tilde_p_int_4}
        p(t) = \left(f_t^{-1}\right)^*p_0 + d\chi_t(q(t)).
    \end{align}

    By definition of the Hamiltonian flow $\phi_t$, we have
    \begin{align*}
        \phi_t(q_0,p_0) = \left(f_t(q_0), (f_t^{-1})^*p_0 + d\chi_t\left(f_t(q_0)\right) \right), 
    \end{align*}
    which by definition of $\tau_{d\chi_t}$ and $f^\sharp_t$, we have the desired equality:
    \[
    \phi_t = \tau_{d\chi_t}\circ f_t^\sharp.
    \]
\end{proof}

\begin{remark}[Interpretation of linear-in-momentum]
The Hamiltonian of Equation~\eqref{eqn:linear_p_Hamiltonian}, $H(q,p) = \langle p,A(q)\rangle + V(q)$, in Theorem~\ref{thm:linear_hamiltonian_specialform} is used only as a generating Hamiltonian for a class of symplectomorphisms $\phi_t = \tau_{d\chi_t}\circ f^\sharp_t$ that preserve Lagrangian graphs (hence Legendre structure). Here $p$ is the canonical cotangent (dual) coordinate on $T^*Q$, and $H$ is not meant to represent a physical/mechanical energy of a particle system (in particular, it does not have a quadratic-in-$p$ kinetic term). This linear-in-momentum restriction is used to guarantee the Legendre preserving normal form in Theorem~\ref{thm:preserve_legendre}. The input-side Legendre dynamics described in Section~\ref{sec:domain} is not restricted by the linear-in-momentum condition and may have general (nonlinear) dependence on momentum.
     
\end{remark}

Finally, combining Theorem~\ref{thm:linear_hamiltonian_specialform} with the characterization in Theorem~\ref{thm:preserve_legendre} and the identification of Theorem~\ref{thm:lagrangian_legendre_type}, we obtain our main theorem:
 \begin{theorem}
 \label{thm:main}
 Any SR update arising from an input-driven Hamiltonian that is at most linear in $p$ is a cotangent lift composed with an exact fiber translation. More precisely, for each time step $t$, the SR update map is of the form $\tau_{d\chi_t}\circ f^\sharp_t$ and hence maps $L_\psi$ to $L_{\psi\circ f^{-1}+\chi}$, thereby preserving Legendre dynamics. 
\end{theorem}

Therefore, under the identification of Legendre duality with Lagrangian graphs (Theorem~\ref{thm:lagrangian_legendre_type}), any SR built from such a Hamiltonian provides a representation map that preserves the Legendre dynamics of Section~\ref{sec:domain}.

    


\subsection{Generic linear symplectic reservoirs}

We first recall a generic linear symplectic reservoir obtained by exact
discretization of a quadratic input-dependent Hamiltonian. Let $x := \binom{q}{p} \in \mathbb{R}^{2n}$ denote the state variable representing the generalized position and generalized momentum respectively. Consider \textbf{input-dependent} Hamiltonian:
\begin{align}
    \label{eqn:hamiltonian}
    H(x,u) = \underbrace{\frac{1}{2}x^\top Mx}_{\text{Internal Energy}} - \underbrace{x^\top C u}_{\text{Interaction Energy}} ,
\end{align}
where $M = M^\top \in \mathbb{M}_{2n}\left(\mathbb{R}\right)$ is symmetric positive-definite, $C$ is a input-to-state coupling matrix. With the canonical symplectic matrix $J$, Hamilton's equations
give
\[
\dot x = JMx-JCu=:Ax+Bu, \qquad J:= \begin{bmatrix}
        0 & I_n \\
        -I_n & 0
    \end{bmatrix}
\]
Under zero order hold condition on the input over one time step $\Delta t$, the
exact discretization is given by
\[
x_{k+1}=Wx_k+W_{\rm in}u_k,
\]
where
\begin{align}
    \label{eqn:linear_sr_AB}
    W=e^{A\Delta t},\qquad
W_{\rm in}=\int_0^{\Delta t}e^{As}B\,ds.
\end{align}

Since $A\in\mathfrak{sp}(2n)$, we have $W\in \operatorname{Sp}(2n)$. Thus, for each
fixed input $u_k$, the affine map
\[
F_{u_k}(x)=Wx+W_{\rm in}u_k
\]
is an affine symplectomorphism. We summarize formally as follows:
\begin{theorem}
Consider the linear reservoir system obtained from the input-dependent Hamiltonian:
\begin{align*}
        H(x,u) = \frac{1}{2}x^\top Mx - x^\top C u,
    \end{align*}
    with $M$ symmetric positive-definite and $C$ a coupling matrix such that:
    \[
    A = JM,\, B = -JC,\, W = e^{A\Delta t}, \, W_{\rm in}=\int_0^{\Delta t}e^{As}B\,ds.
    \]

    Then for every input $u_t$, the affine map:
    \[
    F_{u_t}(x)  = Wx + W_{in}u_t
    \]
    is an affine symplectomorphism of $\left(\mathbb{R}^{2n}, \omega\right)$ with
    \[
    F^*_u \omega(v,w) = \omega(Wv,Ww).
    \]

\end{theorem}

\begin{proof}

   By construction, the translation term $W_{\operatorname{in}} u$ is independent of $x$, the derivative of $F_u$ with respect to $x$ is
    \[
    D_xF_u=W,
    \]
 Hence
\[
F_u^*\omega(v,w)=\omega(D_xF_u v,D_xF_u w)
=\omega(Wv,Ww).
\]
Since $W\in \operatorname{Sp}(2n)$, we have $\omega(Wv,Ww)=\omega(v,w)$. Therefore
$F_u^*\omega=\omega$ and $F_u$ is an affine symplectomorphism.

\end{proof}


The construction above gives a generic linear symplectic reservoir. It is included
to separate the notion of symplecticity from Legendre preservation: a generic affine symplectomorphism need not preserve Legendre graphs. The Legendre preserving normal form of Theorem~\ref{thm:linear_hamiltonian_specialform} and Theorem~\ref{thm:main} is obtained by imposing the stronger Hamiltonian condition that the Hamiltonian be at most linear in the momentum. We now specialize to this Legendre-preserving subclass.

\subsection{Linear Realization of Symplectic Reservoir}

\label{sec:linear-sr}

Now with this necessary and sufficient condition we are ready to construct a first example: a linear symplectic reservoir from Hamiltonian mechanics that preserves the Legendre structure of the input. The following ``linear realization'' serves as a baseline as it is linear in both the momentum $p$ and in the base variable $q$. The nonlinear HSR of the next subsection replaces this linear base dynamics by a bounded nonlinear velocity field.

Consider input-dependent Hamiltonian given by
\begin{align*}
    H(q,p,u) = p^\top Sq + \frac{1}{2}q^\top Lq  - \left(q^\top C_q u + p^\top C_p u \right),
\end{align*}
where $L = L^\top,S\in \mathbb{M}_{n\times n}\left(\mathbb{R}\right)$, $ C_q,C_p \in \mathbb{M}_{n\times m}\left(\mathbb{R}\right)$ are coupling matrices that define how the input influences the state's energy.  By Hamilton's equations,
\begin{align*}
    \frac{d}{dt}q &= \frac{\partial H}{\partial p} = Sq - C_pu\\
    \frac{d}{dt}p &= -\frac{\partial H}{\partial q} = -\left[Lq +S^\top p  - C_q u\right].
\end{align*}
Hence in matrix form, the time evolution of the system is given by
\begin{align}
    \label{eqn:lti_hamiltonian_linear_p}
    \frac{dx(t)}{dt} = \frac{d}{dt}\begin{bmatrix}
        q \\ p
    \end{bmatrix}
     = \begin{bmatrix}
         S & 0 \\
         -L & -S^\top
     \end{bmatrix}
     \begin{bmatrix}
        q \\ p
    \end{bmatrix}
    + 
    \begin{bmatrix}
        -C_p \\ C_q
    \end{bmatrix}u =: Ax+Bu, 
\end{align}
where we set
\begin{align*}
    A:= \begin{bmatrix}
        S & 0 \\
        -L & -S^\top
    \end{bmatrix}, \quad 
    B:=\begin{bmatrix}
        -C_p \\ C_q
    \end{bmatrix}.
\end{align*}

By construction $A \in \mathfrak{sp}(2n,\mathbb{R})$ since
\begin{align*}
A^\top J &= \begin{bmatrix}
    S^\top & -L \\
    0 & -S
\end{bmatrix}
\begin{bmatrix}
    0 & I \\
    -I & 0
\end{bmatrix} = 
\begin{bmatrix}
    L & S^\top \\
    S & 0
\end{bmatrix} \\
JA &= \begin{bmatrix}
    0 & I \\
    -I & 0
\end{bmatrix}\begin{bmatrix}
    S & 0 \\
    -L & -S^\top
\end{bmatrix} = \begin{bmatrix}
    -L & -S^\top \\
    -S & 0
\end{bmatrix}
\end{align*}
Hence $A^\top J + JA = 0$ and $A \in \mathfrak{sp}(2n,\mathbb{R})$. We therefore obtain the linear Hamiltonian SR update in the form of Equation~\eqref{eqn:linear_sr_AB} with

\begin{align}
\label{eqn:src_linear_p}
    x_{k+1} = Wx_k + W_{\operatorname{in}}u_k, \quad \text{where }
    \begin{cases}
        W := e^{A\Delta t} \\
        W_{\operatorname{in}} = \int_0^{\Delta t}e^{As}B\,ds.
    \end{cases}
\end{align}
When $A$ is invertible, we have $W_{\rm in} = A^{-1} \left(e^{A\Delta t} - I\right)B$. The linear realization above is symplectic and lies in the linear-in-momentum Hamiltonian class required for Legendre preservation. Its generator has block form
\[
A=
\begin{pmatrix}
S&0\\
-L&-S^\top
\end{pmatrix},
\]
so its eigenvalues are those of $S$ and $-S^\top$. Thus, for generic choices of $S$, the matrix $e^{A\Delta t}$ has reciprocal expanding and contracting directions, and the undriven trajectory can exhibit growth along expanding directions. This fixed matrix state iteration mechanism motivates the nonlinear realization below.

\subsection{Nonlinear Hamiltonian Symplectic Reservoir}
\label{sec:nonlinear-sr}

We now introduce a nonlinear Hamiltonian realization of the Legendre preserving normal form of Theorem~\ref{thm:preserve_legendre}. The construction applies Theorem~\ref{thm:linear_hamiltonian_specialform}
to a Hamiltonian that is linear in the momentum but nonlinear in the
configuration variable. In particular, the nonlinearity enters through the configuration velocity field, rather than through an arbitrary activation applied to the full phase-space state. This keeps the update in the class of Legendre-preserving symplectomorphisms.

Throughout this subsection, the symbols $C,W,W_{\operatorname{in}}$ are local to the nonlinear parameterization and are unrelated to the matrices used in the linear realization of Section~\ref{sec:linear-sr}.

Recall Equation~\eqref{eqn:linear_p_Hamiltonian} from  Theorem~\ref{thm:linear_hamiltonian_specialform} and consider the parameterized family


\begin{definition}
\label{defn:nhsr}
Let $Q=\mathbb{R}^n$ and $u\in\mathbb{R}^m$. A \textbf{nonlinear Hamiltonian Symplectic Reservoir (N-HSR)} is the input driven Hamiltonian system on $T^*Q$ generated by
\begin{equation}
\label{eq:nlsr-hamiltonian}
H(q, p, u) = \langle p,\, \mathcal{A}(q, u)\rangle + V(q,u)
\end{equation}
where
\[ \mathcal{A} (q,u)
=
C \tanh(W  q+ W_{\operatorname{in}}  u )
\]
is a  smooth bounded configuration space velocity field and $V :Q\times\mathbb{R}^m\to\mathbb{R}$ is a smooth scalar potential. $\tanh : \mathbb{R}^{n_{\operatorname{int}}} \to \mathbb{R}^{n_{\operatorname{int}}}$ acts componentwise  and the linear maps
\[
C : \mathbb{R}^{n_{\operatorname{int}}} \to \mathbb{R}^n, \qquad
W : \mathbb{R}^n \to \mathbb{R}^{n_{\operatorname{int}}}, \qquad
W_{\operatorname{in}} : \mathbb{R}^m \to \mathbb{R}^{n_{\operatorname{int}}}
\]
are parameters of the family. Here $n_{\rm int}\ge 1$ denotes the dimension of the intermediate space on which the componentwise nonlinearity acts, $n$ is the configuration dimension, and $m$ is the input dimension. 
\end{definition}


By construction, the configuration space velocity field $\mathcal{A} : Q \times \mathbb{R}^m \to TQ$ is $C^\infty$ and globally Lipschitz in~$q$. Let $\norm{\cdot }_{\operatorname{op}}$ denote the operator norm, for any $u \in \mathbb{R}^m$ and $q, q' \in Q$, 

\begin{equation}\label{eq:nlsr-lipschitz}
\norm{\mathcal{A}(q, u) - \mathcal{A}(q', u)} \leq L_{\mathcal{A}} \norm{q - q'}, \qquad
L_\mathcal{A} := \norm{C}_{\operatorname{op}} \cdot\norm{W}_{\operatorname{op}}.
\end{equation}

This is due to the inequality:
\begin{align*}
    \norm{\mathcal{A}(q,u) - \mathcal{A}(q',u)} \leq \sup_{z \in Q} \norm{D_q \mathcal{A}(z, u)}_{\operatorname{op}} \cdot \norm{q - q'}.
\end{align*}
The Jacobian over $q$ given by
\begin{align}
\label{eqn:CW_La}
D_q \mathcal{A}(q, u) = C\cdot\operatorname{diag}\left(\operatorname{sech}^2(Wq + W_{\operatorname{in}} u)\right)\cdot W,
\end{align}
hence $\norm{D_q \mathcal{A}(q, u)}_{\operatorname{op}} \leq \norm{C}_{\operatorname{op}} \cdot\norm{W}_{\operatorname{op}}$ for all $(q, u)$, since $\norm{\operatorname{diag}(\operatorname{sech}^2(\cdot))}_{\operatorname{op}} \leq 1$.

Thus the Hamilton's equations for~\eqref{eq:nlsr-hamiltonian} becomes
\begin{align}
\frac{d}{dt} q &= \mathcal{A}(q, u) = C\cdot\tanh(Wq + W_{\operatorname{in}} u),
\label{eq:nlsr-qdot}\\
\frac{d}{dt} p &= -\left(D_q \mathcal{A}(q, u)\right)^\top p -\nabla_q V(q,u). 
\label{eq:nlsr-pdot}
\end{align}

By Theorem~\ref{thm:linear_hamiltonian_specialform}, the N-HSR preserves Legendre dynamics, more formally:
 
\begin{theorem}
\label{thm:nlsr-legendre}
For each fixed $u \in \mathbb{R}^m$, the $\Delta t = 1$ Hamiltonian flow $\varphi_u : T^*Q \to T^*Q$ of the N-HSR (Definition~\eqref{defn:nhsr}) is a Legendre graph preserving symplectomorphism given by the normal form
\[
\varphi_u=\tau_{d\chi_u}\circ f_u^\sharp,
\]
where $f_u$ is the time-one flow of
\[
\dot q= \mathcal{A}(q,u),
\]
and
\[
\chi_u(q)
=
-\int_0^1 V(f_{\tau-1,u}(q),u)\,d\tau.
\]
Consequently,
\[
\varphi_u(L_\psi)=L_{\psi\circ f_u^{-1}+\chi_u}.
\]
\end{theorem}

\begin{proof}
For fixed $u$, the Hamiltonian (Equation~\eqref{eq:nlsr-hamiltonian}) is at most linear in $p$. By Equation~\eqref{eq:nlsr-lipschitz}, for each $u \in \mathbb{R}^m$, the configuration space velocity field $\mathcal{A}(\cdot, u)$ is globally Lipschitz, hence its integral curves exist for all $t \in \mathbb{R}$ by Picard--Lindel\"of  \cite{hartman1964}, and so $\mathcal{A}(\cdot, u)$ is a complete vector field. 

The $\Delta t= 1$ flow $f_u : Q \rightarrow Q$ is therefore a $C^\infty$ diffeomorphism (with inverse being the $\Delta t= 1$  flow of~$-\mathcal{A}(\cdot, u)$,  also globally Lipschitz and hence complete). The hypotheses of Theorem~\ref{thm:linear_hamiltonian_specialform} are thus satisfied by $H_u(q,p) := H(q,p,u)$.
     Explicitly, we have
\begin{align*}
    f^\sharp_u(q, p) = \left(f_u(q), (Df_u(q))^{-\top} p\right), 
\end{align*}

The expression for $\chi_u$ follows from Equation~\eqref{eqn:chi_t} to the $\Delta t = 1$ flow with fixed input $u$. The final equation follows from Theorem~\ref{thm:main}.

\end{proof}

\subsubsection{Discrete-time implementation}
\label{sec:nlsr-discrete}

For an input sequence $\{u_k\}$, the continuous-time update
$\varphi_{u_k}$ is implemented by first computing a numerical approximation
$\hat f_{u_k}$ of the $\Delta t= 1$ configuration flow (e.g. integrating the configuration ODE (Equation~\eqref{eq:nlsr-qdot}) numerically). The corresponding normal-form lift is
\begin{align*}
q_{k+1} &= \hat f_{u_k}(q_k) \approx f_{u_k}(q_k),\\
p_{k+1} &= D\hat f_{u_k}(q_k)^{-\top}p_k + d\hat\chi_{u_k}(q_{k+1}).
\end{align*}

When $V\equiv0$, this reduces to the cotangent lift
\[
p_{k+1}=D\hat f_{u_k}(q_k)^{-\top}p_k.
\]
Thus numerical error changes the realized diffeomorphism $\hat f_u$,
but the graph-preservation property holds for the realized map whenever
the fibre update is implemented in the above lifted normal forms.

\subsubsection{Finite-horizon growth bounds}
\label{sec:nlsr-boundedness}
 
Theorem~\ref{thm:nlsr-legendre} establishes Legendre graph preservation of nonlinear Hamiltonian SR.  The linear realization of Section~\ref{sec:linear-sr} can exhibit growth along expanding directions since its generator has block form
\[
A=\begin{pmatrix}S&0\\-L&-S^\top\end{pmatrix},
\]
so $e^{A\Delta t}$ generically has reciprocal expanding and contracting directions. Here we show that the reservoir cotangent bundle state $(q, p)$ of nonlinear Hamiltonian SR of Equation~\eqref{eq:nlsr-hamiltonian} satisfies explicit finite-horizon growth bounds. In particular, the N-HSR architecture of Equation~\eqref{eq:nlsr-hamiltonian} gives explicit finite-horizon growth bounds for both $q$ and $p$ without imposing a spectral radius contraction condition. 

For the finite-horizon growth estimate we specialize to the pure cotangent-lift case $V\equiv0$, which is also the version used in the numerical experiments.

The full $V \neq 0$ case still preserves Legendre graphs by Theorem~\ref{thm:nlsr-legendre}, where its momentum equation contains an additional forcing term involving $\nabla_qV$. We omit this extension here because the purpose of the estimate is to contrast the nonlinear N-HSR realization with the fixed-matrix growth mechanism of the linear SR realization.

\begin{proposition}
\label{prop:nlsr-boundedness}
For any $u \in \mathbb{R}^m$, any initial condition $(q_0, p_0) \in L_{\psi_0}$ with $\psi_0 \in C^2$, and any $t \geq 0$. Suppose $V\equiv 0$, then the continuous-time trajectory
of Equations~\eqref{eq:nlsr-qdot}--\eqref{eq:nlsr-pdot} satisfies
\[
\norm{q(t) - q_0} \leq t \cdot \norm{C}_{\operatorname{op}}\sqrt{n_{\operatorname{int}}},
\quad
\norm{p(t)} \leq e^{L_\mathcal{A}\, t}\,\norm{d\psi_0(q_0)},
\]
where $L_\mathcal{A} = \norm{C}_{\operatorname{op}}\cdot\norm{W}_{\operatorname{op}}$. In the discretized case, after $k$ successive applications of the time-$1$ flow (each with a possibly different input~$u_j$), $\|q_k - q_0\| \leq k\,\|C\|_{\operatorname{op}}\,\sqrt{n_{\operatorname{int}}}$
and $\|p_k\| \leq e^{L_\mathcal{A}\, k}\,\|d\psi_0(q_0)\|$.
\end{proposition}
 
\begin{proof}
\begin{description}
    \item[Configuration space bound on $q$:] 

    Since $\vert \tanh(z_i)\vert \leq 1$ componentwise, $\norm{\tanh(z)}_2\leq \sqrt{n_{\operatorname{int}}}$ for any $z\in\mathbb{R}^{n_{\operatorname{int}}}$. Hence 
    \[
    \norm{\frac{d}{dt} q(t)} = \norm{C\cdot \tanh(\cdot)} \leq \norm{C}_{\operatorname{op}} \sqrt{n_{\operatorname{int}}},
    \]
    integrating over $[0, t]$ gives the desired bound. The $k$-step bound follows immediately from triangle inequality.

    \item[Momentum space bound on $p$:]

    By Theorem~\ref{thm:linear_hamiltonian_specialform}, at time $t$ with $V \equiv 0$, the time-$t$ flow is the cotangent lift $\left(f_u^t\right)^\sharp$, hence 
    \[
    p(t) = \Phi(t)^{-\top} p_0,
    \]
    where $\Phi(t)$ is a solution to $\frac{d}{dt}\Phi = D_q \mathcal{A}(q(t), u) \Phi$ with $\Phi(0) = I$. By Liouville's formula, $\det\Phi(t) = \det\Phi(0)\cdot \exp\bigl(\int_0^t \operatorname{tr}\,D_q \mathcal{A}(q(s),u)\,ds\bigr)$. Since $\det\Phi(0) = 1$ and the exponential is strictly positive, $\Phi(t)$ is invertible for all $t \geq 0$.

Set $\Psi(t) := \Phi(t)^{-1}$ for all $t$. Differentiating $\Phi\cdot\Psi = I$ gives $\frac{d}{dt}\Psi = -\Psi\,D_q \mathcal{A}(q(t), u)$ with $\Psi(0) = I$.
In integral form, we have
\[
\Psi(t) = I - \int_0^t \Psi(s)\,D_q \mathcal{A}(q(s), u)\,ds.
\]
Taking operator norms on both sides, by triangular inequality we get
\begin{align}
    \label{eqn:prop:nlsr-boundedness:1}
    \norm{\Psi(t)}_{\operatorname{op}}
    \leq 1 + \int_0^t \norm{\Psi(s)}_{\operatorname{op}}
    \norm{D_q \mathcal{A}(q(s), u)}_{\operatorname{op}} ds
    \leq 1 + L_\mathcal{A} \int_0^t \norm{\Psi(s)}_{\operatorname{op}} ds,
\end{align}
where the last equality follows from Equation~\eqref{eqn:CW_La}.
Now notice this is an inequality of the form $u(t) \leq 1 + L_\mathcal{A} \int_0^t u(s)ds$ with $u := \norm{\Psi(\cdot)}_{\operatorname{op}}$, so we can apply Gronwall Bellman inequality (where if $u(t) \leq \alpha + \int_0^t \beta u(s)ds$ for constants $\alpha, \beta \geq 0$, then
$u(t) \leq \alpha \cdot e^{\beta t}$), Equation~\eqref{eqn:prop:nlsr-boundedness:1} becomes
\[
\norm{\Psi(t)}_{\operatorname{op}} \leq e^{L_\mathcal{A}\, t}.
\]
Since $\norm{p(t)} = \norm{\Psi(t)^\top p_0} \leq \norm{\Psi(t)}_{\operatorname{op}}\,\norm{p_0}$, the single step bound follows. For the discrete $k$-step bound, each application of the time-$1$ flow will contribute a factor
$\norm{\Psi_j}_{\operatorname{op}} \leq e^{L_\mathcal{A}}$, and by submultiplicativity of the operator norm gives the desired bound
\[
\norm{p_k} \leq e^{L_\mathcal{A}\, k}\,\norm{d\psi_0(q_0)}.
\]
\end{description}
\end{proof}
 

The bounds of Proposition~\ref{prop:nlsr-boundedness} are finite horizon growth bounds for the pure cotangent-lift case $V \equiv 0$ and they do not imply the classical echo state property or uniform contraction of the  phase space, in the sense that they hold at the continuous level and any numerical implementation preserves the graph structure for the realized map when the fibre update is implemented in lifted normal form.
Indeed, exact symplectic and cotangent-lift maps preserve the Liouville volume form and are generally incompatible with global phase-space contraction. The guaranteed property studied here is instead Legendre-graph preservation. The nonlinear HSR supplements this structural guarantee with explicit finite-horizon growth estimates for bounded inputs. The estimates control the size of trajectories on any given finite time interval. They are not asymptotic boundedness or contraction statements: the constants in these bounds depend explicitly on the design parameters $\norm{C}_{\rm op}$ and $\norm{W}_{\rm op}$. $\norm{C}_{\rm op}$ sets the linear configuration drift and $\norm{C}_{\rm op}\cdot \norm{W}_{\rm op}$ sets the momentum growth rate, so the finite-horizon growth can be controlled at the architectural level.

Recent reservoir computing theory also treats the echo state property, fading memory, state and input forgetting, and solution stability as distinct notions  \cite{ortega2025echoes, ortega2026stochastic}. The estimates above are finite-horizon growth estimates for the Legendre-preserving representation, not a proof of classical ESP.

\subsubsection{Comparison with the eigenvalue structure of the linear
realization}
\label{sec:nlsr-eigenvalue-comparison}

For the linear SR realization of Section~\ref{sec:linear-sr}, the update has the form
\[
x_{k+1}=Wx_k+W_{\operatorname{in}}u_k,
\qquad
W=e^{A\Delta t},
\]
where
\[
A=
\begin{pmatrix}
S&0\\
-L&-S^\top
\end{pmatrix}.
\]
so the eigenvalues of $e^{A\Delta t}$ occur in reciprocal expanding
and contracting directions for generic $S$. This is the fixed-matrix
state-iteration mechanism that the nonlinear N-HSR avoids. 


The nonlinear HSR is different, where its state update is
\[
x_{k+1}=\varphi_{u_k}(x_k),\qquad x_k=(q_k,p_k),
\]
where
\[
\varphi_{u_k}=\tau_{d\chi_{u_k}}\circ f_{u_k}^{\sharp}.
\]
The derivative
\[
D\varphi_{u_k}(z_k)
\]
acts on {infinitesimal perturbations} $(\delta q, \delta p)$ around the current state, not on the state $x_k$ itself.  For the cotangent lift part,
\[
D(f_u^\sharp)(q,p)
=
\begin{pmatrix}
Df_u(q) & 0\\
* & Df_u(q)^{-\top}
\end{pmatrix}.
\]
Thus whilst the local reciprocal pair structure remains, it describes linearized sensitivity rather than direct state iteration. Over multiple time steps the linearized dynamics is a product of state and input dependent matrices,
\[
D\Phi_k(z_0)
=
D\varphi_{u_{k-1}}(z_{k-1})\cdots D\varphi_{u_0}(z_0),
\]
rather than $W^k$. Hence the reciprocal-pair spectrum of each local
Jacobian does not imply repeated multiplication of the state by one
unstable matrix. Indeed, the finite-horizon bounds in Proposition~\ref{prop:nlsr-boundedness} show how the nonlinear HSR avoids the fixed matrix state iteration mechanism.


\begin{remark}[Expressivity within the normal-form class]

By the ``if and only if'' argument of Theorem~\ref{thm:preserve_legendre}, the maps $\tau_{d\chi}\circ f^\sharp$ are precisely the Legendre-preserving symplectomorphisms, so this class is the ``complete'' design space for Legendre-preserving updates.  The expressivity statement established in this paper is therefore within the Legendre-preserving normal-form class; below we discuss how the linear-in-momentum Hamiltonians realize maps within it.
The two components of the normal form can be generated by Hamiltonian blocks:
\begin{enumerate}
    \item If
\[
H_A(q,p)=\langle p,A(q)\rangle,
\]
then its $\Delta t= 1$ flow is the cotangent lift $f^\sharp$ of the base flow of $A$.
\item If
\[
H_\chi(q,p)=-\chi(q),
\]
then its $\Delta t= 1$ flow is the exact fiber translation
\[
(q,p)\mapsto(q,p+d\chi(q)).
\]
\end{enumerate}

More generally, in the Hamiltonian
\[
H(q,p)=\langle p,A(q)\rangle+V(q),
\]
the scalar potential $V$ generates the exact fibre-translation term through the function $\chi$ obtained by integrating $V$ along the
base flow, as in Theorem~\ref{thm:nlsr-legendre}. Hence finite compositions of these two types of blocks realize maps of the form
\[
\tau_{d\chi}\circ f^\sharp.
\]


On compact subsets, neural network parameterizations of $A$ and $V$ can approximate smooth vector fields and scalar potentials, respectively.

Since ODE flows depend continuously on their generating vector fields, N-HSR provides an approximation mechanism within the
Legendre preserving normal-form class. A full universality theorem for
fading-memory input-output filters is beyond the scope of this paper.

\end{remark}

}

 { 
\section{Numerical illustrations}
\label{sec:exp}
We conclude the paper with numerical illustrations of the structural
Legendre preservation results. These experiments verify that the implemented N-HSR realizes the normal-form identities at the representation level. 

For reproducibility of the experiments, all experiments are CPU-based and are performed on Apple M3 Max with 128GB of RAM. The source code and data of the numerical analysis are openly available at \texttt{https://github.com/rsimonfong/Symplectic\_Legendre}.

Throughout the section we focus on the pure cotangent lift case $V\equiv0$, for which the discrete update has the form (see Section~\ref{sec:nlsr-discrete})
\[
q_{k+1}=\hat f_{u_k}(q_k),
\qquad
p_{k+1}=D\hat f_{u_k}(q_k)^{-\top}p_k.
\]
This is the $V\equiv0$ version of the normal form in Theorem~\ref{thm:nlsr-legendre}. 

The experiments verify two points:
\begin{enumerate}
    \item \emph{Local stepwise structure:} the state-space update satisfies the geometric identities predicted by the normal form (Section~\ref{sec:E1});
    \item \emph{Global trajectory-wise structure:} repeated application of the cotangent-lift update transports Legendre duality correctly along
    an OU-driven trajectory (Section~\ref{sec:E4}).
\end{enumerate}




In both experiments, the input sequence $\{u_k\}$ is generated by a stable Ornstein--Uhlenbeck (OU) process and is used only as a controlled stochastic drive for the state-space dynamics. We evaluate the resulting cotangent-bundle trajectories $(q_k,p_k)\in T^*Q$. No task readout is trained in these experiments since the purpose is to test the representation map itself.
Operational evaluation on a supervised downstream task would instead probe the readout, which lies outside the representation-level scope of this paper (Remark~\ref{rmk:readout}). We therefore restrict attention to the structural identities and task-level benchmarking is left for future work.


\subsection{Geometric defect verification}
\label{sec:E1}
The first experiment verifies that the implemented N-HSR update realizes the geometric structure described by the main normal-form results. For each random seed and architecture, we generate the same OU input sequence and sample trajectory points produced by the corresponding driven update. At each sampled point, we evaluate the residuals of the one-step identities predicted by the normal-form results from Theorems~\ref{thm:preserve_legendre},
\ref{thm:linear_hamiltonian_specialform}, and \ref{thm:nlsr-legendre}, together with the discrete lifted implementation of Section~\ref{sec:nlsr-discrete}. Let
\[
\hat{F}_u:(q,p) \mapsto (q_+,p_+)
\]
denote one numerical state space update on $T^*Q$. We write its Jacobian in block form:
\[
D\hat{F}_u(z)=
\begin{pmatrix}
D_q q_+ & D_p q_+\\
D_q p_+ & D_p p_+
\end{pmatrix}.
\]
The first identity comes from symplecticity. By Theorem~\ref{thm:nlsr-legendre}, N-HSR update is a Hamiltonian symplectomorphism, so its Jacobian should satisfy
\[
D\hat F_u(z)^\top J D\hat F_u(z)=J.
\]
We measure the residual of this identity by the symplectic defect
\[
\operatorname{SD}
=
\frac{\|D\hat F_u(z)^\top J D\hat F_u(z)-J\|_F}{\|J\|_F}.
\]

The second identity comes from the normal form in Theorem~\ref{thm:preserve_legendre}
\[
\hat F_u=\tau_{d\hat\chi_u}\circ \hat f_u^\sharp.
\]
In this form, the base coordinate evolves by
\[
q_+=\hat f_u(q),
\]
so it is independent of the incoming covector $p$. Therefore, we should have
\[
D_p q_+=0.
\]
We measure the residual of this identity by the normal-form defect
\[
\operatorname{NFD}
=
\frac{\|D_p q_+\|_F}{1+\|D\hat F_u(z)\|_F}.
\]

The third identity is on the cotangent-lift covector update. In the full normal form, the fibre update includes the exact translation $d\hat\chi_u$. In the experiments in this section we use the pure cotangent-lift case $V\equiv0$, so $d\hat\chi_u=0$. Recall the discrete lifted implementation in Section~\ref{sec:nlsr-discrete} and Theorem~\ref{thm:nlsr-legendre}, the fibre update satisfies
\[
p_+
=
D\hat f_u(q)^{-\top}p.
\]
We measure the residual of this identity by the cotangent lift defect
\[
\operatorname{CLD}
=
\frac{\|p_+-D\hat f_u(q)^{-\top}p\|}
     {1+\|p_+\|}.
\]

The numerator in each defect is zero exactly when the corresponding theoretical identity is satisfied. These quantities are averaged over sampled trajectory points and 20 random seeds. 

Table~\ref{tab:e1-geometric-defects} reports the result of this experiment where we compare three models. For each seed, all systems are driven by the same OU input sequence. The N-HSR trajectory is initialized from $z_0=(q_0,p_0)$ with $p_0=q_0$, corresponding to the quadratic Legendre graph $\psi(q)=\frac{1}{2}q^\top P q$, where $P = I$. The Linear SR uses the corresponding embedded cotangent bundle initialization, whereas the Orthogonal ESN uses the standard zero initial state. We observe 

\begin{enumerate}
    \item N-HSR ($V\equiv 0$) update has zero cotangent lift defect and zero normal-form defect up to numerical precision, with symplectic defect at the finite-difference precision floor.
    \item The Linear SR diagnostic has near zero SD and NFD, but diverges in all tested seeds. This shows that symplecticity and block normal-form structure alone do not give the finite horizon behavior of the N-HSR realization.
    \item In contrast, the scaled orthogonal ESN baseline, using a \texttt{tanh} activation with $W_{\operatorname{ESN}}=0.9\cdot Q_{\rm orth}$ and $Q_{\rm orth}$ orthogonal, has order-one geometric defects. This shows that orthogonality or contractive spectral scaling does not enforce the cotangent lift normal-form identities, and therefore does not enforce Legendre-graph preservation.
\end{enumerate}
These results confirm that the implemented N-HSR realizes the Legendre preserving normal form, rather than merely producing a generic recurrent feature map.

The Linear SR baseline is included as a diagnostic: it is symplectic by construction, but it is not the nonlinear cotangent-lift realization used in Theorem~\ref{thm:nlsr-legendre}. The ESN baseline is included to test whether standard contractive or orthogonal reservoir structure enforces the same geometric identities.

\begin{table}[htbp]
\centering
\small
\caption{Geometric defect verification over 20 seeds. Finite-difference Jacobians use central differences with $\epsilon_{\rm fd}=10^{-4}$ and \texttt{JAC\_STRIDE}=10; the N-HSR symplectic defect is at the finite-difference precision floor. Orthogonal ESN  uses a \texttt{tanh} activation with $W=0.9\cdot Q_{\rm orth}$, where $Q_{\rm orth}$ is orthogonal. Entries displayed as $0.0$ were zero to the reported numerical precision.}
\label{tab:e1-geometric-defects}
\begin{tabular}{@{}lcccc@{}}
\toprule
Architecture & CLD & SD& NFD & Diverged \\
\midrule
N-HSR ($V\equiv 0$)
  & $0.0  $
  & $1.68{\times}10^{-12} \pm 2.77{\times}10^{-13}$
  & $0.0  $
  & $0/20$ \\
Linear SR
  & $3.55{\times}10^{-1} \pm 1.43{\times}10^{-1}$
  & $8.48{\times}10^{-16} \pm 1.07{\times}10^{-16}$
  & $0.0  $
  & $20/20$ \\
Orthogonal ESN
  & $8.50{\times}10^{-1} \pm 9.53{\times}10^{-2}$
  & $1.16{\times}10^{0} \pm 3.35{\times}10^{-2}$
  & $5.02{\times}10^{-1} \pm 1.65{\times}10^{-2}$
  & $0/20$ \\
\bottomrule
\end{tabular}
\end{table}

\subsection{OU-driven transported Legendre graph}
\label{sec:E4}

The previous experiment checks the stepwise cotangent-lift identity from step $k$ to step $k+1$ along sampled OU-driven input trajectory points:
\[
p_{k+1}=D\hat f_{u_k}(q_k)^{-\top}p_k.
\]
We now test this identity across the accumulated trajectory. 

For each random seed, we generate a stable OU input sequence $\{u_k\}$ and use it to
drive the same base dynamics
\[
q_{k+1}=\hat f_{u_k}(q_k).
\]
We follow one trajectory from an initial Legendre graph and ask whether the covector remains on the transported Legendre graph over time. In particular, experiments below use this same base trajectory and differ only in how the covector $p_k$ is updated.

We initialize the state on a quadratic Legendre graph
\[
L_{\psi_0}
=
\{(q,p):p=d\psi_0(q)\},
\qquad
\psi_0(q)=\frac{1}{2} q^\top Pq,
\]
where $P$ is chosen to be a diagonal non-identity matrix. This avoids the Euclidean self-dual case $p=q$, so the defect tests covector transport rather than equality of coordinates. The initial point is thus given by $p_0=d\psi_0(q_0)=Pq_0$. Let
\[
F_k
:=
\hat f_{u_{k-1}}\circ\cdots\circ \hat f_{u_0}
\]
be the accumulated base update map, meaning $q_k=F_k(q_0)$. Set
\[
G_k :=DF_k(q_0)
\]
to be its corresponding accumulated Jacobian. Since $V\equiv0$ by construction, the transported Legendre graph after $k$ steps is
\[
L_{\psi_k},
\qquad
\psi_k=\psi_0\circ F_k^{-1}.
\]
Therefore the correct dual coordinate at time $k$ is
\[
p_k^\star
=
d\psi_k(q_k)
=
DF_k(q_0)^{-\top}p_0
=
G_k^{-\top}p_0.
\]
Equivalently, transported Legendre consistency is characterized by
\[
G_k^\top p_k=p_0.
\]
We therefore define the pullback residual over the trajectory as follows
\[
\delta_{\mathrm{pull}}(k)
=
\frac{\|G_k^\top p_k-p_0\|}
     {1+\|p_0\|}.
\]
This residual vanishes exactly when $p_k$ is the covector obtained by
transporting $p_0$ along the accumulated base map.

We compare the N-HSR cotangent-lift update with three variations using the same base trajectory: a frozen covector update, a wrong-Jacobian covector update, and a corrupted cotangent lift. 
The N-HSR update uses the cotangent-lift rule
\[
p_{k+1}=D\hat f_{u_k}(q_k)^{-\top}p_k.
\]
The frozen-covector variation uses
\[
p_{k+1}=p_k,
\]
which tests what happens when the dual coordinate is not transported.
The wrong-Jacobian variation uses the forward Jacobian instead of the
inverse transpose,
\[
p_{k+1}=D\hat f_{u_k}(q_k)p_k,
\]
which treats a covector as if it were a tangent vector. The corrupted
cotangent-lift variation applies the correct update and then perturbs the
covector,
\[
p_{k+1}
=
D\hat f_{u_k}(q_k)^{-\top}p_k+\vartheta_k,
\qquad
\vartheta_k\sim\mathcal N(0,\sigma_{\rm corr}^2 I).
\]

Table~\ref{tab:e4-ou-lift} 
shows that the N-HSR maintains transported Legendre consistency to
numerical precision, whereas the other variants accumulate nonzero transported graph error. This illustrates the operational role of the inverse-transpose
covector update: the base trajectory alone is not enough to preserve Legendre duality.

\begin{table}[htbp]
\centering
\small
\caption{ OU-driven transported Legendre-graph consistency over 20 seeds. The primary metric is
$\delta_{\mathrm{pull}}(k)= \|G_k^\top p_k-p_0\|/(1+\|p_0\|)$, where $G_k$ is the accumulated base map Jacobian along the driven trajectory. Lower values indicate better transported graph consistency.
}
\label{tab:e4-ou-lift}
\begin{tabular}{@{}lcc@{}}
\toprule
Method & Mean $\delta_{\operatorname{pull}}$ & Max $\delta_{\operatorname{pull}}$ \\
\midrule
N-HSR ($V \equiv 0$)
  & $8.43{\times}10^{-16} \pm 5.90{\times}10^{-16}$
  & $1.73{\times}10^{-15} \pm 9.82{\times}10^{-16}$ \\
Frozen covector
  & $1.56{\times}10^{-1} \pm 1.24{\times}10^{-1}$
  & $3.48{\times}10^{-1} \pm 3.22{\times}10^{-1}$ \\
Wrong covector transport
  & $4.62{\times}10^{-1} \pm 6.96{\times}10^{-1}$
  & $1.54{\times}10^{0} \pm 2.82{\times}10^{0}$
   \\
Corrupted cotangent lift
  & $3.04{\times}10^{-1} \pm 1.06{\times}10^{-1}$
  & $7.00{\times}10^{-1} \pm 4.61{\times}10^{-1}$ \\
\bottomrule
\end{tabular}
\end{table}

This residual $\delta_{\rm pull}$ differs from the CLD in Section~\ref{sec:E1}. The CLD is a stepwise residual:
\[
\operatorname{CLD}_k
=
\frac{\|p_{k+1}-D\hat f_{u_k}(q_k)^{-\top}p_k\|}
     {1+\|p_{k+1}\|}.
\]
On the other hand, the pullback residual here in Section~\ref{sec:E4} is an accumulated trajectory residual: it checks whether $p_k$ equals the covector obtained by transporting $p_0$ through the full composition
\[
F_k=\hat f_{u_{k-1}}\circ\cdots\circ\hat f_{u_0}.
\]

In this sense, Section~\ref{sec:E1} verifies the local update rule, whereas Section~\ref{sec:E4} verifies that repeated application of that same rule transports the Legendre graph along a driven trajectory.

Together, Sections~\ref{sec:E1} and~\ref{sec:E4} show that the N-HSR does not merely produce a symplectic recurrent trajectory. It realizes the cotangent-lift normal form predicted by the results of the paper and transports the dual coordinate according to the inverse-transpose identity. 

This is the representation level property targeted by the paper: Legendre duality is preserved by the internal cotangent bundle state, without requiring a trained task readout.
}
\clearpage

\appendix
\section{Ornstein-Uhlenbeck process is Legendre dynamics}
\label{app:OU_dynamics}

In this section we use Markov semigroups to show that Ornstein-Uhlenbeck dynamics is an example of Legendre dynamics. For the rest of the section we restrict to exponential families.

\begin{definition}
    The \textbf{Markov semigroup $\left(P_t\right)_{t\geq 0}$ associated with a Markov process $\left(X_t\right)_{t\geq 0}$} is given by the following condition: For every bounded measurable function $f:\Omega\rightarrow\mathbb{R}$, for all $x\in \Omega$ and $t\geq 0$, $P_t$ satisfies
    \[
    P_tf(x) = \mathbb{E}\left[f(X_t)\, \middle|\, X_0 = x \right], \quad x\in \Omega.
    \]
\end{definition}

Let $\mathfrak{L}$ denote the infinitesimal generator of Markov semigroup $\left(P_t\right)_{t\geq 0}$ defined by:
\[
\mathfrak{L}f:=\lim_{t\rightarrow 0^+} \frac{P_t f - f}{t}.
\]
Then the group of adjoint operators acting on measures, denoted by $\left(P_t^*\right)$, has corresponding adjoint generator $\mathfrak{L}^*$.

\begin{lemma}
    \label{lemma:markov_lemma}
    Let $\left(P_t\right)_{t\geq 0}$ be a Markov semigroup with generator $\mathfrak{L}$. Fix $t>0$, for any $f$ in the domain of $\mathfrak{L}$ the map $u(s) :=P_{t-s}f\in C^1(\Omega)$ satisfies
    \[
    \frac{\partial}{\partial s}u(s) = -\mathfrak{L}u(s).
    \]
\end{lemma}

\begin{proof}
    By definition of Markov semigroup,
    \begin{align*}
        \frac{u(s+h) - u(s)}{h} &= \frac{P_{t-s-h}f - P_{t-s}f}{h} = \frac{P_{t-s-h}f - P_{t-s-h}P_hf}{h} \\
        &= P_{t-s-h}\frac{\left(I - P_h\right)f}{h} = - P_{t-s-h}\frac{\left(P_h - I\right)f}{h}.
    \end{align*}
    Then as $h\rightarrow 0$, by definition of generator $\mathfrak{L}$, we have the desired equality as follows:
    \begin{align*}
        \frac{d}{ds}u(s) &= \lim_{h\rightarrow 0} \frac{u(s+h) - u(s)}{h}= -P_{t-s} \lim_{h\rightarrow 0} \frac{P_h f- f}{h}  \\
        &= -P_{t-s} \mathfrak{L}f = -\mathfrak{L} (P_{t-s}f) = -\mathfrak{L}u(s).
    \end{align*}
    
\end{proof}

Let $\operatorname{Dom}(\mathfrak{L})$ denote the operator domain of $\mathfrak{L}$ and let $C_0(\Omega)$ denote the space of continuous functions vanishing at infinity (see for example  \cite{kurt2022markov}).

\begin{theorem}
\label{thm:markov_tangent}
Let $E = \{p_\theta(u) = \exp(\theta \cdot T(u) - \psi(\theta))\}$ be an exponential family on a state space $\Omega$. Assume that $a:\Xi\to\mathbb R^n$ is locally Lipschitz and forward complete on $\Xi$, then a Markov semigroup $\left(P_t\right)_{t\geq 0}$ on $C_0(\Omega)$ preserves exponential family if and only if the corresponding generator $\mathfrak{L}$ satisfies, for all $\theta \in \Xi$ and $p_\theta \in \operatorname{Dom}(\mathfrak{L^*})$ the following
    \begin{align}
    \label{eqn:L_affine_T}
        \frac{(\mathfrak{L}^*p_\theta)(u)}{p_\theta(u)}=a(\theta) \cdot  T(u)+b(\theta) \quad\text{with}\quad b(\theta)=-\bar{\eta}(\theta) \cdot  a(\theta). 
    \end{align}
    Equivalently, $\mathfrak{L}^*p_\theta$ is an affine function in $T$ under $p_\theta$.
\end{theorem}

\begin{proof}
      \textbf{($\Rightarrow$)}
      Suppose exponential family $E$ is preserved by Markov semigroup $\left(P_t\right)_{t\geq 0}$. Let $\theta = \theta(0) \in \Xi$ be arbitrary starting point of a curve $\theta(t) \in \Xi$. Then by assumption the following path is also in $E$, given by
      \[
      p(t) = p_{\theta(t)} = P_t^* p_\theta \in E.
      \]
      Here $(P_t^*)_{t\ge 0}$ denotes the adjoint semigroup acting on densities
with respect to the reference measure $\mu$, defined by
\[
  \langle P_t^* p, f\rangle = \langle p, P_t f\rangle,
  \qquad \forall\, f\in C_0(\Omega),
\]
where $\langle\cdot,\cdot\rangle$ denotes integration against $\mu$. Differentiating at $t = 0$ with $\mathfrak{L}^*$, we get
      \[
      \mathfrak{L}^* p_{\theta} = \frac{d}{dt} p_{\theta(t)}\Big\vert_{t = 0}
      \]

      Since $ p_{\theta(t)} \in E$ for all $t$, by definition of exponential family,
      \begin{align*}
          \frac{d}{dt} p_{\theta(t)}(u) &= \frac{d}{dt} \exp\left[\theta(t)\cdot T(u)- \psi\left(\theta(t)\right)\right] \\
          \text{(Chain rule)}&= p_{\theta(t)}\frac{d}{dt} \left[\theta(t)\cdot T(u)- \psi\left(\theta(t)\right)\right] \\
          &= p_{\theta(t)} \left[\frac{d}{dt}\theta(t)\cdot T(u)- \frac{d}{dt}\psi\left(\theta(t)\right)\right] \\
          &= p_{\theta(t)} \left[\frac{d}{dt}\theta(t)\cdot T(u)- \frac{d}{dt}\theta(t) \cdot \underbrace{\nabla_\theta\psi(\theta(t))}_{\bar{\eta}(t)}\right] \\
          &= p_{\theta(t)} \left[\underbrace{\frac{d}{dt}\theta(t)}_{=:a(\theta(t)) \in \mathbb{R}^n}\cdot T(u)- \underbrace{\frac{d}{dt}\theta(t) \cdot \bar{\eta}(t)}_{=: -b(\theta(t))\in \mathbb{R}}\right] .
      \end{align*}

      Hence,
      \begin{align*}
          \frac{d}{dt} p_{\theta(t)}(u) = p_{\theta(t)} (u) \left[a(\theta(t))\cdot T(u) + b(\theta(t))\right].
      \end{align*}
      By construction $\theta(0) = \theta$, and we get the desired generator identity Equation~\eqref{eqn:L_affine_T} as follows
      \begin{align*}
          \mathfrak{L}^*p_{\theta}(u) = \frac{d}{dt} p_{\theta(t)}(u)\Big\vert_{t = 0}= p_{\theta} (u) \left[a(\theta)\cdot T(u) + b(\theta)\right].
      \end{align*}

    \textbf{($\Leftarrow$)} Conversely suppose Equation~\eqref{eqn:L_affine_T} is satisfied. Considering the following differential equation on $\Xi$
    \[
    \dot{\theta}(t) = a(\theta(t)), \qquad \theta(0) = \theta.
    \]
    By assumption $a:\Xi\to\mathbb R^n$ is locally Lipschitz and forward complete, hence the above ODE admits a unique solution $\theta(t)$ for all $t\geq 0$ and remain in $\Xi$.
    Let $q_t(u) := p_{\theta(t)}(u)\in E$, then by construction $\log q_t(u) = \theta(t)\cdot T(u) - \psi(\theta(t))$. By the derivation of the first part, we obtain
    \begin{align*}
        \frac{d}{dt}q_t(u) = q_t(u) \frac{d}{dt} \log q_t(u) &= q_t(u)\cdot\left[\frac{d}{dt}\theta(t)\cdot T(u)- \frac{d}{dt}\theta(t) \cdot \nabla_\theta\psi(\theta(t))\right] \\
        &= q_t(u) \cdot \left[a(\theta(t))\cdot T(u) + b(\theta(t))\right].
    \end{align*}
    By assumption since Equation~\eqref{eqn:L_affine_T} is satisfied at $\theta(t)$, we obtain
    \begin{align*}
        \frac{\mathfrak{L}^*q_t(u)}{q_t(u)} = a(\theta(t)) \cdot  T(u)+b(\theta(t)).
    \end{align*}
    
    Combining the two equations above, we get
    \begin{align}
    \label{eqn:dtqt_L*qt}
        \partial_tq_t = \frac{d}{dt}q_t = \mathfrak{L}^*q_t.
    \end{align}

    Let $\mu$ be reference measure on $\Omega$ and set $\mu_t = q_t\mu$. 

    Fix $t>0$, for $f\in \operatorname{Dom}(\mathfrak{L})$ define the function on $u\in \Omega$:
    \[
    \nu_s(u) = \left(P_{t-s}f\right)(u).
    \]
    By construction it satisfies the hypothesis of Lemma~\ref{lemma:markov_lemma}, hence:
    \begin{align}
    \label{eqn:markov_lemma_in_theorem}
        \frac{d}{ds}\nu_s = -\mathfrak{L}\nu_s.
    \end{align}
    Now consider the following
    \begin{align*}
        \Phi(s) := \langle \mu_s, \nu_s\rangle = \int_{\Omega} \nu_s(u)q_s(u) \mu(du) =  \int_{\Omega} \nu_s(u)q_s(u) d\mu .
    \end{align*}
    Using Equations~\eqref{eqn:dtqt_L*qt} and~\eqref{eqn:markov_lemma_in_theorem}, we get
    \begin{align*}
        \Phi'(s) &= \int_\Omega \partial_s\nu_sq_sd\mu + \int_\Omega\nu_s\partial_sq_sd\mu\\
        &= \int_\Omega (-\mathfrak{L}\nu_s)q_sd\mu + \int_\Omega\nu_s \mathfrak{L}^*q_sd\mu \\
        &= -\langle \mu_s,\mathfrak{L} \nu_s\rangle + \langle \mathfrak{L}^*\mu_s, \nu_s\rangle \\
        &= -\langle \mu_s,\mathfrak{L} \nu_s\rangle + \langle \mu_s, \mathfrak{L}\nu_s\rangle = 0.
    \end{align*}
    Hence $\Phi(s)$ is a constant function  on $0\leq s\leq t$ and $\Phi(t) = \Phi(0)$ for all fixed $t\in \mathbb{R}_{\geq 0}$. In particular, since $\nu_s(u) = \left(P_{t-s}f\right) (u)$, we have
    \[
    \begin{cases}
        \nu_0(u) = \left(P_{t}f\right) (u)\\
        \nu_t(u) = f (u).
    \end{cases}
    \]
    Hence we have the following
    \begin{align*}
        \int_\Omega f(u) q_t(u) d\mu &= \Phi(t) \\
        &=\Phi(0) = \int_\Omega\nu_0(u)q_0(u)d\mu = \int_\Omega \left(P_{t}f\right) (u)q_0(u)d\mu.
    \end{align*}
    Set $\mu_0 = q_0\mu$ and $\mu_t = q_t\mu$, the above equation becomes
    \begin{align}
    \label{eqn:equal_inner}
    \langle\mu_t,f\rangle = \langle \mu_0,P_tf\rangle, \quad \forall f\in \operatorname{Dom}(\mathfrak{L}).
    \end{align}
     Consider a sequence of functions $\left\{f_n\right\}_{n\geq 0} \subset \operatorname{Dom}(\mathfrak{L})$ such that $f_n\rightarrow f$. Since $P_t$ is bounded in $C_0$, by dominating convergence we have
    \begin{align*}
        \begin{cases}
            \langle\mu_t,f_n\rangle\rightarrow\langle\mu_t,f\rangle\\
            \langle \mu_0,P_tf_n\rangle \rightarrow \langle \mu_0,P_tf\rangle
        \end{cases}
    \end{align*}
    By  \cite{kurt2022markov}[Lemma 3.1], $\operatorname{Dom}(\mathfrak{L})$ is dense in $C_0(\Omega)$, we can extend Equation~\eqref{eqn:equal_inner} to all $f\in C_0(\Omega)$, in particular,
    \begin{align*}
        \langle\mu_t,f\rangle = \langle \mu_0,P_tf\rangle , \quad \forall f\in C_0(\Omega).
    \end{align*}
    
    Since $\langle\mu_0,P_t f\rangle = \langle P_t^*\mu_0,f\rangle$ by the
definition of the adjoint semigroup $(P_t^*)_{t\ge 0}$, Equation~\eqref{eqn:equal_inner} can be
rewritten as
\[
  \langle\mu_t,f\rangle = \langle P_t^*\mu_0,f\rangle,
  \qquad \forall\, f\in C_0(\Omega).
\]
By the Riesz--Markov--Kakutani representation theorem, the map
$\mu\mapsto(f\mapsto\langle\mu,f\rangle)$ is injective on finite, signed Borel
measures on $\Omega$, hence 
\[
\mu_t = P_t^* \mu_0.
\]

    Finally we unravel the constructions. By construction of $\mu_0,\mu_t$, the above equality becomes
    \[
    \mu_t = q_t\mu = P_t^*q_0\mu.
    \]
    Then by construction of $q_0 = p_\theta, q_t = p_{\theta(t)}$, we have
    \[
    P_t^*\left(p_\theta\mu\right) = p_{\theta(t)}\mu \in E.
    \]
    Hence $P_t$ preserves $E$ as desired. 
\end{proof}

\begin{remark}[Relation to information geometry]
    In the context of Information Geometry, Theorem~\ref{thm:markov_tangent} is a semigroup formulation of the principle “a flow stays on a statistical manifold iff its velocity is tangent to it.” For the exponential family
    \begin{align*}
        E &:= \left\{p_\theta(x) = \exp\left(\theta\cdot T(x) - \psi(\theta)\right) \right\} \\
        T_{p_\theta}E &:= \left\{p_\theta\left(a \cdot \left( T -\bar{\eta}(\theta) \right)\right) \,\middle|\, a\in \mathbb{R}^k\right\}.
    \end{align*}

    The \textbf{only-if, $\left(\Rightarrow\right)$} direction essentially means if the forward flow stays within the manifold (flow preserves an embedded submanifold), then the generator $\mathfrak{L}$ (the differential of $p_\theta$) is tangent. This is a standard result in differential geometry applied to probabilistic setting where the state space is $C_0(\Omega)^*$ via  the Riesz-Markov-Kakutani theorem's identification of finite measures in  $C_0(\Omega)^*$ with $C_0(\Omega)$.

    Conversely, in  \cite{pistone2009projecting} the authors set $\alpha = \frac{\mathfrak{L}^*p}{p}$ as the Fokker-Planck tangent in exponential coordinates and show that if $\alpha \in \operatorname{span}\left\{1,T\right\}$, then the solution to Fokker-Planck equation starting in a finite-dimensional exponential family evolves inside it. 

    Hence the \textbf{if, $\left(\Leftarrow\right)$} direction of Theorem~\ref{thm:markov_tangent} can be understood as a semigroup formulation of  \cite{pistone2009projecting}[Proposition 8(iii)]. We include the present proof since it is novel in the sense that:
    \begin{itemize}
        \item our proof relies only on semigroup calculus and backward-forward duality (without the machinery of Information Geometry).
        \item In the literature the result is scattered across multiple references. The present proof is shorter and self-contained. 
    \end{itemize}
\end{remark}

We are now ready to formally define Legendre dynamics. For the rest of the paper we focus on exponential family.

\begin{definition}[Legendre dynamics in semigroup form]
\label{defn:legendre_semigroup}
Let $E = \{p_\theta(u) = \exp(\theta \cdot T(u) - \psi(\theta))\}$ be an exponential family on a state space $\Omega \ni u$ with parameter manifold $\Xi \ni \theta$ and dual map $\bar{\eta} = \nabla \psi(\theta)$. A Markov semigroup $(P_t)_{t\geq 0}$ is \textbf{Legendre type on $E$} if for every $\theta \in \Xi$ and $t\geq 0$ there exists $\theta(t) \in \Xi$ with $\theta(0) = \theta$ such that
\[
P_t^*p_\theta= p_{\theta(t)} \in E.
\]
\end{definition}

The geometric definition of Legendre dynamics and the semigroup definition are equivalent along the path generated by $P_t^*$. The former is geometric and the latter is on stochastic processes. A process is of Legendre type if and only if its adjoint semigroup preserves the exponential family if and only if the lifted evolution on $T^*\Xi$ keeps the Lagrangian submanifold $\mathfrak{L}_\psi$ invariant. Furthermore, by Theorem~\ref{thm:markov_tangent} this holds when $\frac{\mathfrak{L}^*p_\theta}{p_\theta}$ is an affine function on sufficient statistics under $p_\theta$. We write the equivalence of statistical,  geometrical, and infinitesimal formulations of Legendre dynamics more formally as follows.

 {

\begin{theorem}
\label{thm:equivalent}
     Let $Q = \Xi$, $\psi:\Xi\rightarrow \mathbb{R}$ be $C^2$. Let $L_\psi$ denote the Lagrangian submanifold of Legendre type corresponding to potential $\psi$, given by\footnote{See Definition~\ref{defn:legendre_duality}, Theorem~\ref{thm:lagrangian_legendre_type}.}
     \[
    L_\psi := \left\{(\theta,\bar{\eta}) \in T^*Q\middle| \bar{\eta} = d\psi(\theta)\right\}.
    \]
   Let $(P_t)_{t\geq 0}$ be a Markov semigroup with adjoint generator $\mathfrak{L}^*$ defined on minimal regular exponential family 
   \[
   E=\{p_\theta(u)=\exp(\theta\cdot T(u)-\psi(\theta)):\theta\in\Xi\}.
   \]
   For each $\theta_0\in \Xi$, the following are equivalent. Hence whenever any (hence all) of them hold, they determine  a unique $C^1$ curve $\theta(t)\in \Xi$ with $\theta(0) = \theta_0$. 
    \begin{itemize}
        \item (Semigroup Legendre Dynamics) For all $t\geq 0$, there exists a curve $\theta(t)\in \Xi$ with $\theta(0) = \theta_0 \in \Xi$ such that
        \[
        P^*_tp_{\theta_0} = p_{\theta(t)} \in E.
        \]
        \item (Lagrangian submanifold of Legendre type) For $\theta \in \Xi$ and the curve $\theta(t)$ above, the semigroup-induced density path $P_t^*p_{\theta_0}$ is represented by a curve satisfying for all $t\geq 0$
        \[
        t\mapsto \left(\theta(t),\bar{\eta}(t)\right) \in L_\psi, \qquad \bar{\eta}(t) = \nabla\psi(\theta(t)),
        \]
        with $\theta(0) = \theta_0$ and $P_t^*p_{\theta_0}=p_{\theta(t)}$. 
        \item (Infinitesimal characterization Theorem~\ref{thm:markov_tangent}) There exists a locally Lipschitz, forward-complete vector field $a(\theta)$ such that for all $\theta\in \Xi$:
         \begin{align*}
        \frac{(\mathfrak{L}^*p_\theta)(u)}{p_\theta(u)}=a(\theta) \cdot  T(u)+b(\theta) \quad\text{with}\quad b(\theta)=-\bar{\eta}(\theta) \cdot  a(\theta),
    \end{align*}
    and $\theta(t)$ solves $\dot{\theta}(t) = a(\theta(t))$ with $\theta(0) = \theta_0$ with solution defined on $\Xi$ for all $t \geq 0$.
    \end{itemize}
\end{theorem}

\begin{proof}
The proof is straightforward by definition:
\begin{itemize}
    \item $(i)\implies(iii)$:  follows from the \textbf{only-if} $\left(\Rightarrow\right)$ direction of the proof of Theorem~\ref{thm:markov_tangent}.
    \item $(iii)\implies(i)$: follows from the \textbf{if} $\left(\Leftarrow\right)$ direction of the proof of Theorem~\ref{thm:markov_tangent}.
    \item $(i)\implies(ii)$: if $P_t^*p_{\theta_0}=p_{\theta(t)}$, then the corresponding dual coordinate $\bar\eta(t)=d\psi(\theta(t))$ gives a lifted curve in $L_\psi$.
    \item $(ii)\implies(i)$: conversely by $(ii)$ the lifted curve represents the semigroup-induced density path, and hence implies (i).
\end{itemize}

        

\end{proof}

}

Now we can show that OU dynamics satisfies the hypothesis of Theorem~\ref{thm:markov_tangent}, thus by Theorem~\ref{thm:equivalent}, it is a Legendre dynamics.


Recall for the Gaussian family we write the natural parameter as $\theta=(\xi,\Lambda)$, and reserve $\bar\eta(\theta)=\nabla\psi(\theta)$ for the expectation coordinate.

\begin{theorem}
Let $\left(P_t\right)_{t\geq 0}$ be the Markov semigroup of the Ornstein-Uhlenbeck SDE on $\mathbb{R}^d$:
\[
dU_t = K(\mu-U_t)dt + \sigma dW_t,
\]
where $K,\sigma,\mu$ are constants, $W_t$ denote a Wiener process, and $A:=\sigma\sigma^\top \geq 0$. Let $\mathfrak{L}$ denote the generator of $\left(P_t\right)_{t\geq 0}$ with adjoint $\mathfrak{L}^*$. Let $E = \left\{p_\theta \middle| \theta\in\Xi \right\}$ be minimal, regular Gaussian exponential family on $\mathbb{R}^d$. Then for all $\theta\in \Xi$, there exists curve $\theta(t) \in \Xi$ with $\theta(0) = \theta$ such that
         \begin{align*}
        \frac{(\mathfrak{L}^*p_\theta)(u)}{p_\theta(u)}=a(\theta) \cdot  T(u)+b(\theta) \quad\text{with}\quad b(\theta)=-\bar \eta(\theta) \cdot  a(\theta),
    \end{align*}
    and $\theta(t)$ solves $\dot{\theta}(t) = a(\theta(t))$ with $\theta(0) = \theta$. The induced vector field $a(\theta)$ is locally Lipschitz and forward complete. Hence condition $(iii)$ of Theorem~\ref{thm:equivalent} holds.
\end{theorem}

\begin{proof}
    For Ito diffusion $dX_t = b(X_t)dt + \sigma(X_t)dW_t$, the backwards and forward generators are given by  \cite{pavliotis2014stochastic} as
    \begin{align*}
        \mathfrak{L}f &= b(x)\cdot \nabla f + \frac{1}{2}\sum_{i,j}^d \left(\sigma\sigma^\top\right)_{ij}\frac{\partial^2}{\partial x_i \partial x_j}f\\
        \mathfrak{L}^*p &= \frac{\partial p}{\partial t} = \nabla\cdot\left(-b(x)p + \frac{1}{2}\nabla\cdot \left(\sigma\sigma^\top\right)(x) p\right)
    \end{align*}

    Let $U_t \in \mathbb{R}^d$ solve Ornstein-Uhlenbeck (OU) SDE:
    \[
    dU_t = K(\mu-U_t)dt + \sigma dW_t,
    \]
    where $\sigma,\mu$ are constants, $W_t$ denote a Wiener process. Then as we have defined in the theorem statement, we have $A = \left[A_{ij}\right]:=\sigma\sigma^\top \geq 0$ is a \textit{constant} and $b(u) = K(\mu - u)$, we obtain
    \begin{align*}
        &\frac{1}{2}\sum_{i,j}^d \left(\sigma\sigma^\top\right)_{ij}\frac{\partial^2}{\partial x_i \partial x_j}f = \frac{1}{2}\sum_{i,j}^d A_{ij}\frac{\partial^2}{\partial x_i \partial x_j}f \\
        & \quad= \frac{1}{2}\sum_{i,j}^d A_{ij}\frac{\partial^2}{\partial x_j \partial x_i}f =  \frac{1}{2}\sum_{i,j}^d A_{ij}H_{ji}f = \frac{1}{2}\operatorname{tr}\left(A\nabla^2 f\right),
    \end{align*}
    where $H = \left[H_{ij}\right]$ denote the Hessian matrix and the second last equality is due to symmetry of second derivatives. Moreover, 
    \begin{align*}
        &\nabla\cdot\nabla\cdot \left(\sigma\sigma^\top p\right) = \sum_{i,j}^d \frac{\partial^2}{\partial x_i \partial x_j} A_{ij} p \\
        & \quad = \sum_{i,j}^d\left[A_{ij} \frac{\partial^2}{\partial x_i \partial x_j} p  + \left(\frac{\partial}{\partial x_i}A_{ij}\right) \frac{\partial}{\partial x_j} p  
        + \left(\frac{\partial}{\partial x_j}A_{ij}\right) \frac{\partial}{\partial x_i} p + \left(\frac{\partial^2}{\partial x_i \partial x_j} A_{ij}\right) p
        \right] \\
        & \quad = \operatorname{tr}\left(A\nabla^2 p \right) + \operatorname{div}A\cdot\nabla p + \operatorname{div}A^\top\cdot\nabla p + \left(\nabla\cdot \nabla\cdot A\right) p \\
        & \quad = \operatorname{tr}\left(A\nabla^2 p \right) + 0 + 0 + 0 = \operatorname{tr}\left(A\nabla^2 p \right).
    \end{align*}
    The last three term vanish since $A$ is a constant. Hence, on the OU SDE, we have    
    \begin{align}
        \mathfrak{L}f &= K(\mu - u)\cdot \nabla f + \frac{1}{2}\operatorname{tr}\left(A\nabla^2f\right) \nonumber\\
        \mathfrak{L}^*p &= \frac{\partial p}{\partial t} = -\nabla \cdot \left(K(\mu - u)p \right)+ \frac{1}{2}\operatorname{tr}\left(A\nabla^2p\right). \label{eqn:OU_L_2}
    \end{align}

    Let $p>0$ be $C^2(\Omega)$ and let $\ell = \log(p)$. Since $\nabla p = p \nabla \ell$, we have $\nabla^2 p  = p\left( \nabla^2\ell  + \left(\nabla\ell\right)\left(\nabla\ell\right)^\top\right)$. Hence, substituting $b:= K(\mu-u)$ for simplicity, dividing $p$ from Equation~\ref{eqn:OU_L_2} and we have
    \begin{align}
    \label{eqn:eqn:OU_L_2_new}
        \frac{\mathfrak{L}^*p}{p} = - \nabla\cdot b - b\cdot \nabla \ell + \frac{1}{2}\operatorname{tr} \left(A\nabla^2 \ell\right) + \frac{1}{2}\left(\nabla\ell\right)^\top A \nabla \ell.
    \end{align}

    By assumption since $E$ is a minimal regular Gaussian exponential family under natural parametrization $\theta = \left(\xi, \Lambda\right) := \left(\xi, \Sigma^{-1}\right)$, $p = p_\theta \in E$ can be expressed as  \cite{jordan2003introduction}[Equation (13.6)] \footnote{Here we replace their notation of $a$ by $-\psi$,  $x$ by $u$, $\Lambda = \Sigma^{-1}$ for consistency with our previous notation. We choose $T(u) = (u,-\frac{1}{2}uu^\top)$ such that it has the form.}
    \[
    p_\theta(u) = \exp\left\{\xi^\top u - \frac{1}{2} u^\top \Lambda u - \psi(\xi,\Lambda)\right\}, \quad T(u) = \left(u, -\frac{1}{2}uu^\top\right).
    \]
    \begin{align*}
        \nabla\ell &= \nabla \log(p_\theta) = \xi - \Lambda u \\
        \nabla^2 \ell &= -\Lambda.
    \end{align*}
    Now we substitute the above and $b(u) = K(\mu-u)$ back into each term of Equation~\eqref{eqn:eqn:OU_L_2_new}. The first term, by construction,
    \begin{align*}
        \nabla\cdot b = \sum_{i=1}^d \frac{\partial}{\partial u_i} b_i(u) = -\sum_{i=1}^d K_{ii} = -\operatorname{tr}(K).
    \end{align*}
    Since $\Lambda = \Sigma^{-1}$ is symmetric. The second term is given by
    \begin{align*}
        -b\cdot \nabla \ell &= - (K\mu-Ku)^\top(\xi - \Lambda u) \\
        &= -\xi^\top K\mu + u^\top \Lambda K\mu + u^\top K^\top \xi - u^\top K^\top \Lambda u.
    \end{align*}
    The third term is
    \begin{align*}
        \frac{1}{2} \operatorname{tr}\left(A\nabla^2 \ell\right) = -\frac{1}{2} \operatorname{tr}\left(A\Lambda\right).
    \end{align*}
    Finally the fourth term becomes
    \begin{align*}
        \frac{1}{2} \left(\nabla\ell\right)^\top A \nabla\ell &= 
         \frac{1}{2}\left(\xi - \Lambda u\right)^\top A \left(\xi - \Lambda u\right) \\
         &=
        \frac{1}{2}\xi^\top A\xi - u^\top \Lambda A \xi + \frac{1}{2} u^\top \Lambda A\Lambda u.
    \end{align*}

    Equation~\eqref{eqn:eqn:OU_L_2_new} thus becomes
    \begin{align*}
        \frac{\mathfrak{L}^*p_\theta}{p_\theta} &= \operatorname{tr}(K) -\xi^\top K\mu + u^\top \Lambda K\mu + u^\top K^\top \xi - u^\top K^\top \Lambda u \\
        & \quad \quad -\frac{1}{2} \operatorname{tr}\left(A\Lambda\right) + 
        \frac{1}{2}\xi^\top A\xi - u^\top \Lambda A \xi + \frac{1}{2} u^\top \Lambda A\Lambda u\\
        &= \left(\operatorname{tr}(K)-\xi^\top K\mu -\frac{1}{2} \operatorname{tr}\left(A\Lambda\right) + \frac{1}{2}\xi^\top A\xi\right) \\
        & \quad \quad + u^\top\left( K^\top \xi + \Lambda K\mu - \Lambda A\xi \right) \\
        & \quad \quad + u^\top \left(\frac{1}{2}\Lambda A\Lambda - K^\top \Lambda\right)u.
    \end{align*}
    The last equality is just regrouping the terms into constant, linear terms and quadratic terms in $u$ respectively. Recall from elementary linear algebra: $u^\top M u = \left\langle \frac{1}{2}\left(M+M^\top\right),uu^\top \right\rangle$. Since both $A,\Lambda$ are symmetric,  the quadratic term becomes
    \begin{align*}
        u^\top \left(\frac{1}{2}\Lambda A\Lambda - K^\top \Lambda\right)u &= 
        \left\langle \frac{1}{2}\left(\left(\frac{1}{2}\Lambda A\Lambda - K^\top \Lambda\right)+\left(\frac{1}{2}\Lambda A\Lambda - K^\top \Lambda\right)^\top\right),uu^\top \right\rangle \\
        &= \left\langle \frac{1}{2}\left(\Lambda A\Lambda - K^\top \Lambda - \Lambda K\right), uu^\top\right\rangle \\
        &= \left\langle K^\top\Lambda + \Lambda K - \Lambda A\Lambda , -\frac{1}{2}uu^\top\right\rangle.
    \end{align*}
    From the above we extract from the linear term and quadratic terms the following
    \begin{align*}
        a_{\xi}(\xi,\Lambda) &= K^\top \xi + \Lambda K\mu - \Lambda A \xi\\
        a_{\Lambda}(\xi,\Lambda) &= K^\top \Lambda + \Lambda K - \Lambda A \Lambda.
    \end{align*}
    Set $b(\xi,\Lambda):= \left(\operatorname{tr}(K)-\xi^\top K\mu -\frac{1}{2} \operatorname{tr}\left(A\Lambda\right) + \frac{1}{2}\xi^\top A\xi\right)$. Then by our choice of sufficient statistics $(T_1(u),T_2(u):= T(u) =  \left(u, -\frac{1}{2}uu^\top\right)$, Equation~\eqref{eqn:eqn:OU_L_2_new} becomes
    \begin{align}
    \label{eqn:OU_L_2_new_2}
        \frac{\mathfrak{L}^*p_\theta}{p_\theta} &= a_{\xi}(\xi,\Lambda)^\top u + \left\langle a_{\Lambda}(\xi,\Lambda), -\frac{1}{2}uu^\top \right\rangle + b(\xi,\Lambda) \nonumber \\
        &= a(\xi,\Lambda) \cdot T(u) + b(\xi,\Lambda)
    \end{align}
    where  $a(\xi,\Lambda) := \left( a_{\xi}(\xi,\Lambda), a_{\Lambda}(\xi,\Lambda)\right)$.

    Note that $a(\xi,\Lambda) = (a_\xi(\xi,\Lambda),a_\Lambda(\xi,\Lambda))$
is polynomial in $(\xi,\Lambda)$, hence $C^1$ on $\Xi$ and therefore locally Lipschitz. Hence for every initial value
$\theta\in\Xi$, the ODE $\dot\theta(t)=a(\theta(t))$, $\theta(0)=\theta$, admits a local solution $\theta(t)\in\Xi$.

 {
For the OU process, the induced parameter curve is in fact global: Let the initial Gaussian law have mean $m_0$ and covariance $\Sigma_0\succ0$, so that
\[
\Lambda_0=\Sigma_0^{-1},\qquad \xi_0=\Lambda_0m_0.
\]
The OU moment equations becomes
\[
m_t=\mu+e^{-Kt}(m_0-\mu), \qquad \Sigma_t
=
e^{-Kt}\Sigma_0e^{-K^\top t}
+
\int_0^t e^{-Kr}A e^{-K^\top r}\,dr .
\]

For every finite \(t\ge0\), the first term is positive definite and the
integral term is positive semidefinite. Since $\Sigma_t\succ0$, so 
\[
\Lambda_t=\Sigma_t^{-1},\qquad \xi_t=\Lambda_t m_t
\]
exist for all \(t\ge0\). Therefore the induced vector field
\(a(\xi,\Lambda)\) is forward complete on $\Xi$ and 
$\theta(t)=(\xi_t,\Lambda_t)\in\Xi$ for all \(t\ge0\).
}

    Since $\partial_t = \mathfrak{L}^*p$, we have that $\frac{d}{dt}\int_{\mathbb{R}^d} p dx = \int_{\mathbb{R}^d} \mathfrak{L}^* pdx = 0$  \cite{pavliotis2014stochastic}. Therefore,
    \begin{align*}
        0= \int_{\mathbb{R}^d}\mathfrak{L}^*p_\theta(u) du = \int_{\mathbb{R}^d}\frac{\mathfrak{L}^*p_\theta(u)}{p_\theta(u)}p_\theta(u) du =  \mathbb{E}_\theta\left[\frac{\mathfrak{L}^*p_\theta(u)}{p_\theta(u)}\right].
    \end{align*}
    Since $a(\xi,\Lambda), b(\xi,\Lambda)$ are constant in $u$, by Equation~\eqref{eqn:OU_L_2_new_2}, 
    \begin{align*}
        0 = \mathbb{E}_\theta \left[a(\xi,\Lambda) \cdot T(u) \right] + b(\xi,\Lambda) = a(\xi,\Lambda)\cdot \mathbb{E}_\theta \left[ T(u) \right] + b(\xi,\Lambda).
    \end{align*}
    Equivalently:
    \begin{align*}
        b(\xi,\Lambda) = - a(\xi,\Lambda)\cdot \mathbb{E}_\theta \left[ T(u) \right]  = - \bar \eta(\theta) \cdot a(\xi,\Lambda).
    \end{align*}
    Therefore Equation~\eqref{eqn:OU_L_2_new_2} becomes
    \begin{align*}
        \frac{\mathfrak{L}^*p_\theta}{p_\theta} = a(\xi,\Lambda) \cdot T(u) - \bar \eta(\theta) \cdot a(\xi,\Lambda).
    \end{align*}
    Under natural parametrization $\theta = \left(\xi, \Lambda\right)$ this is precisely the desired equality. 
 \end{proof}

As an immediate corollary of Theorem~\ref{thm:equivalent} and Theorem~\ref{thm:markov_tangent}, we have the following

\begin{corollary} Ornstein-Uhlenbeck process is Legendre dynamics. In particular, minimal regular exponential family $E$ is preserved under the Markov semigroup of the Ornstein-Uhlenbeck SDE on $\mathbb{R}^d$. 
\end{corollary}

\begin{proof}
    This is an immediate corollary of Theorem~\ref{thm:equivalent} and Theorem~\ref{thm:markov_tangent} in Appendix~\ref{app:OU_dynamics}.
\end{proof}

\clearpage

\newcommand{\etalchar}[1]{$^{#1}$}


\begin{thebibliography}{EGBHP24}

\bibitem[AN00]{amari2000methods}
Shun-ichi Amari and Hiroshi Nagaoka.
\newblock {\em Methods of information geometry}, volume 191.
\newblock American Mathematical Soc., 2000.

\bibitem[ARP16]{anselmi2016invariance}
Fabio Anselmi, Lorenzo Rosasco, and Tomaso Poggio.
\newblock On invariance and selectivity in representation learning.
\newblock {\em Information and Inference: A Journal of the IMA}, 5(2):134--158, 2016.

\bibitem[BCV13]{bengio2013representation}
Yoshua Bengio, Aaron Courville, and Pascal Vincent.
\newblock Representation learning: A review and new perspectives.
\newblock {\em IEEE Transactions on Pattern Analysis and Machine Intelligence}, 35(8):1798--1828, 2013.

\bibitem[Bol21]{bollt2021explaining}
Erik Bollt.
\newblock On explaining the surprising success of reservoir computing forecaster of chaos? the universal machine learning dynamical system with contrast to var and dmd.
\newblock {\em Chaos: An Interdisciplinary Journal of Nonlinear Science}, 31(1), 2021.

\bibitem[BP09]{pistone2009projecting}
Damiano Brigo and Giovanni Pistone.
\newblock Projecting the fokker-planck equation onto a finite dimensional exponential family.
\newblock {\em arXiv preprint arXiv:0901.1308}, 2009.

\bibitem[CW16]{cohen2016group}
Taco Cohen and Max Welling.
\newblock Group equivariant convolutional networks.
\newblock In {\em International Conference on Machine Learning}, pages 2990--2999. PMLR, 2016.

\bibitem[CZAB19]{chen2019symplectic}
Zhengdao Chen, Jianyu Zhang, Martin Arjovsky, and L{\'e}on Bottou.
\newblock Symplectic recurrent neural networks.
\newblock {\em arXiv preprint arXiv:1909.13334}, 2019.

\bibitem[DVSM12]{dambre2012information}
Joni Dambre, David Verstraeten, Benjamin Schrauwen, and Serge Massar.
\newblock Information processing capacity of dynamical systems.
\newblock {\em Scientific Reports}, 2(1):514, 2012.

\bibitem[EGBHP24]{eldred2024lie}
Christopher Eldred, Francois Gay-Balmaz, Sofiia Huraka, and Vakhtang Putkaradze.
\newblock Lie--poisson neural networks (lpnets): Data-based computing of hamiltonian systems with symmetries.
\newblock {\em Neural Networks}, 173:106162, 2024.

\bibitem[EGBP25]{eldred2025clpnets}
Christopher Eldred, Fran{\c{c}}ois Gay-Balmaz, and Vakhtang Putkaradze.
\newblock Clpnets: Coupled lie--poisson neural networks for multi-part hamiltonian systems with symmetries.
\newblock {\em Neural Networks}, 189:107441, 2025.

\bibitem[FLT25a]{fong2025linear}
Robert~Simon Fong, Boyu Li, and Peter Tino.
\newblock Linear simple cycle reservoirs at the edge of stability perform fourier decomposition of the input driving signals.
\newblock {\em Chaos}, 35(4):043109, 2025.

\bibitem[FLT25b]{fong2024universality}
Robert~Simon Fong, Boyu Li, and Peter Tino.
\newblock Universality of real minimal complexity reservoir.
\newblock {\em Proceedings of the AAAI Conference on Artificial Intelligence}, 39(16):16622--16629, 2025.

\bibitem[Gal24]{gallicchio2024euler}
Claudio Gallicchio.
\newblock Euler state networks: Non-dissipative reservoir computing.
\newblock {\em Neurocomputing}, 579:127411, 2024.

\bibitem[GO18a]{grigoryeva2018echo}
Lyudmila Grigoryeva and Juan-Pablo Ortega.
\newblock Echo state networks are universal.
\newblock {\em Neural Networks}, 108:495--508, 2018.

\bibitem[GO18b]{Grigoryeva2018}
Lyudmila Grigoryeva and Juan-Pablo Ortega.
\newblock Universal discrete-time reservoir computers with stochastic inputs and linear readouts using non-homogeneous state-affine systems.
\newblock {\em Journal of Machine Learning Research}, 19(24):1--40, 2018.

\bibitem[GO19]{gonon2019reservoir}
Lukas Gonon and Juan-Pablo Ortega.
\newblock Reservoir computing universality with stochastic inputs.
\newblock {\em IEEE Transactions on Neural Networks and Learning Systems}, 31(1):100--112, 2019.

\bibitem[Har02]{hartman1964}
Philip Hartman.
\newblock {\em Ordinary differential equations}.
\newblock SIAM, 2002.

\bibitem[HOY24]{hu2024structure}
Jianyu Hu, Juan-Pablo Ortega, and Daiying Yin.
\newblock A structure-preserving kernel method for learning hamiltonian systems.
\newblock {\em arXiv preprint arXiv:2403.10070}, 2024.

\bibitem[HOY25]{hu2025global}
Jianyu Hu, Juan-Pablo Ortega, and Daiying Yin.
\newblock A global structure-preserving kernel method for the learning of poisson systems.
\newblock {\em Journal of Nonlinear Science}, 35(4):79, 2025.

\bibitem[HS10]{hartikainen2010kalman}
Jouni Hartikainen and Simo S{\"a}rkk{\"a}.
\newblock Kalman filtering and smoothing solutions to temporal gaussian process regression models.
\newblock In {\em 2010 IEEE International Workshop on Machine Learning for Signal Processing}, pages 379--384. IEEE, 2010.

\bibitem[IY17]{inubushi2017reservoir}
Masanobu Inubushi and Kazuyuki Yoshimura.
\newblock Reservoir computing beyond memory-nonlinearity trade-off.
\newblock {\em Scientific Reports}, 7(1):10199, 2017.

\bibitem[Jae01a]{Jaeger2002}
Herbert Jaeger.
\newblock Short term memory in echo state networks.
\newblock 2001.

\bibitem[Jae01b]{Jaeger2001}
Herbert Jaeger.
\newblock The “echo state” approach to analysing and training recurrent neural networks.
\newblock {\em Bonn, Germany: German national research center for information technology GMD technical report}, 148(34):13, 2001.

\bibitem[Jae02]{jaeger2002a}
Herbert Jaeger.
\newblock Tutorial on training recurrent neural networks, covering bppt, rtrl, ekf and the echo state network approach.
\newblock 5(1), 2002.

\bibitem[JH04]{Jaeger2004}
Herbert Jaeger and Harald Haas.
\newblock Harnessing nonlinearity: Predicting chaotic systems and saving energy in wireless communication.
\newblock {\em Science}, 304(5667):78--80, 2004.

\bibitem[Jor03]{jordan2003introduction}
Michael~I Jordan.
\newblock An introduction to probabilistic graphical models, 2003.

\bibitem[JZZ{\etalchar{+}}20]{jin2020sympnets}
Pengzhan Jin, Zhen Zhang, Aiqing Zhu, Yifa Tang, and George~Em Karniadakis.
\newblock Sympnets: Intrinsic structure-preserving symplectic networks for identifying hamiltonian systems.
\newblock {\em Neural Networks}, 132:166--179, 2020.

\bibitem[KF22]{kurt2022markov}
Kevin Kurt and R{\"u}diger Frey.
\newblock Markov-modulated affine processes.
\newblock {\em Stochastic Processes and their Applications}, 153:391--422, 2022.

\bibitem[KG20]{kayhan2020translation}
Osman~Semih Kayhan and Jan C~van Gemert.
\newblock On translation invariance in cnns: Convolutional layers can exploit absolute spatial location.
\newblock In {\em Proceedings of the IEEE/CVF Conference on Computer Vision and Pattern Recognition}, pages 14274--14285, 2020.

\bibitem[Kul13]{kulis2013metric}
Brian Kulis.
\newblock Metric learning: A survey.
\newblock {\em Foundations and Trends® in Machine Learning}, 5(4):287--364, 2013.

\bibitem[LFT24]{li2023simple}
Boyu Li, Robert~Simon Fong, and Peter Tino.
\newblock {Simple Cycle Reservoirs are Universal}.
\newblock {\em Journal of Machine Learning Research}, 25(158):1--28, 2024.

\bibitem[LJ09]{Lukoservicius2009}
Mantas Lukosevicius and Herbert Jaeger.
\newblock Reservoir computing approaches to recurrent neural network training.
\newblock {\em Computer Science Review}, 3(3):127--149, 2009.

\bibitem[MNM02]{Maass2002}
Wolfgang Maass, Thomas Natschl{\"a}ger, and Henry Markram.
\newblock Real-time computing without stable states: A new framework for neural computation based on perturbations.
\newblock {\em Neural Computation}, 14(11):2531--2560, 2002.

\bibitem[Off24]{offen2024machine}
Christian Offen.
\newblock Machine learning of continuous and discrete variational odes with convergence guarantee and uncertainty quantification.
\newblock {\em arXiv preprint arXiv:2404.19626}, 2024.

\bibitem[OR26a]{ortega2025echoes}
Juan-Pablo Ortega and Florian Rossmannek.
\newblock Echoes of the past: A unified perspective on fading memory and echo states.
\newblock {\em Neural Computation}, 38(5):765--782, 2026.

\bibitem[OR26b]{ortega2026stochastic}
Juan-Pablo Ortega and Florian Rossmannek.
\newblock Stochastic dynamics learning with state-space systems.
\newblock {\em Mathematical Models and Methods in Applied Sciences}, 2026.

\bibitem[Pav14]{pavliotis2014stochastic}
Grigorios~A Pavliotis.
\newblock Stochastic processes and applications.
\newblock {\em Texts in applied mathematics}, 60, 2014.

\bibitem[RPK19]{raissi2019physics}
Maziar Raissi, Paris Perdikaris, and George~E Karniadakis.
\newblock Physics-informed neural networks: A deep learning framework for solving forward and inverse problems involving nonlinear partial differential equations.
\newblock {\em Journal of Computational Physics}, 378:686--707, 2019.

\bibitem[RT10]{rodan2010minimum}
Ali Rodan and Peter Ti$\check{\rm n}$o.
\newblock Minimum complexity echo state network.
\newblock {\em IEEE Transactions on Neural Networks}, 22(1):131--144, 2010.

\bibitem[SS23]{sarkka2023bayesian}
Simo S{\"a}rkk{\"a} and Lennart Svensson.
\newblock {\em Bayesian filtering and smoothing}, volume~17.
\newblock Cambridge university press, 2023.

\bibitem[TD01]{Tino2001}
Peter Tino and Georg Dorffner.
\newblock Predicting the future of discrete sequences from fractal representations of the past.
\newblock {\em Machine Learning}, 45(2):187--217, 2001.

\bibitem[Tin20]{Tino_JMLR_2020}
Peter Tino.
\newblock Dynamical systems as temporal feature spaces.
\newblock {\em Journal of Machine Learning Research}, 21(44):1--42, 2020.

\bibitem[TSK{\etalchar{+}}24]{trirat2024universal}
Patara Trirat, Yooju Shin, Junhyeok Kang, Youngeun Nam, Jihye Na, Minyoung Bae, Joeun Kim, Byunghyun Kim, and Jae-Gil Lee.
\newblock Universal time-series representation learning: A survey.
\newblock {\em arXiv preprint arXiv:2401.03717}, 2024.

\bibitem[TYH{\etalchar{+}}19]{tanaka2019recent}
Gouhei Tanaka, Toshiyuki Yamane, Jean~Benoit H{\'e}roux, Ryosho Nakane, Naoki Kanazawa, Seiji Takeda, Hidetoshi Numata, Daiju Nakano, and Akira Hirose.
\newblock Recent advances in physical reservoir computing: A review.
\newblock {\em Neural Networks}, 115:100--123, 2019.

\bibitem[VCdD24]{vaquero2024symmetry}
Miguel Vaquero, Jorge Cort{\'e}s, and David~Mart{\'\i}n de~Diego.
\newblock Symmetry preservation in hamiltonian systems: Simulation and learning.
\newblock {\em Journal of Nonlinear Science}, 34(6):115, 2024.

\end{thebibliography}
\end{document}